\documentclass[journal]{IEEEtran}
\ifCLASSINFOpdf
\else
\fi

\usepackage{graphicx}
\usepackage{subfigure}
\usepackage{cite}
\usepackage{algorithm}
\usepackage{algorithmic}
\usepackage{amssymb,amsmath,amsthm,amsfonts}
\usepackage{float}
\usepackage{subfigure}
\usepackage{multirow}
\usepackage{tabularx}
\usepackage{multirow}
\usepackage{makecell}
\usepackage{enumerate}
\newtheorem{theorem}{Theorem}
\newtheorem{lemma}{Lemma}
\newtheorem{definition}{Definition}
\usepackage{array}
\begin{document}
%
\title{Discrete Multi-modal Hashing with Canonical Views for Robust Mobile Landmark Search}
\author{Lei~Zhu,
	Zi~Huang,
	Xiaobai~Liu,
	Xiangnan He,
	Jingkuan Song,
	Xiaofang Zhou
	\thanks{This work was partially supported by ARC project under Grant DP150103008 and FT130101530.}
	\thanks{L. Zhu, Z. Huang, and X. Zhou are with the School of Information Technology and Electrical
		Engineering, The University of Queensland, Brisbane, QLD 4072, Australia (e-mail: leizhu0608@gmail.com; huang@itee.uq.edu.au; 	zxf@itee.uq.edu.au).}
	}
%
%


%
%

\markboth{IEEE Transactions on Multimedia}%
{Shell \MakeLowercase{\textit{et al.}}: Bare Demo of IEEEtran.cls for IEEE Journals}
%



\maketitle

\begin{abstract}
Mobile landmark search (MLS) recently receives increasing attention for its great practical values. However, it still remains unsolved due to two important challenges. One is high bandwidth consumption of query transmission, and the other is the huge visual variations of query images sent from mobile devices. In this paper, we propose a novel hashing scheme, named as canonical view based discrete multi-modal hashing (CV-DMH), to handle these problems via a novel three-stage learning procedure. First, a submodular function is designed to measure visual representativeness and redundancy of a view set. With it, canonical views, which capture key visual appearances of landmark with limited redundancy, are efficiently discovered with an iterative mining strategy. Second, multi-modal sparse coding is applied to transform visual features from multiple modalities into an intermediate representation. It can robustly and adaptively characterize visual contents of varied landmark images with certain canonical views. Finally, compact binary codes are learned on intermediate representation within a tailored discrete binary embedding model which preserves visual relations of images measured with canonical views and removes the involved noises. In this part, we develop a new augmented Lagrangian multiplier (ALM) based optimization method to directly solve the discrete binary codes. We can not only explicitly deal with the discrete constraint, but also consider the bit-uncorrelated constraint and balance constraint together. The proposed solution can desirably avoid accumulated quantization errors in conventional optimization method which simply adopts two-step ``relaxing+rounding" framework. With CV-DMH, robust visual query processing, low-cost of query transmission, and fast search process are simultaneously supported. Experiments on real world landmark datasets demonstrate the superior performance of CV-DMH over several state-of-the-art methods.
\end{abstract}


%
\IEEEpeerreviewmaketitle

\section{Introduction}
%
%
%
%
%
%

\IEEEPARstart{W}{ith} the rapid growth of advanced mobile devices and social networking services, tremendous
amount of landmark images have been generated and disseminated in popular
social networks. Mobile landmark search (MLS)
is gaining its importance and increasingly becomes one of the most important
techniques to pervasively and intelligently access knowledge about the landmarks of interest
\cite{sigpro/ChengS16,DBLP:conf/mmm/ChengRSM13}.

However, MLS still remains unsolved due to two important challenges. 1) \emph{Low-bit query transmission}.
Mobile devices have limited computational power, relatively small memory storage, and short battery lifetime.
Consequently, a client-server architecture is the main stream search paradigms adopted in
existing landmark search systems, where query is captured and submitted by mobile devices, computation-intensive landmark search
is performed on the remote sever with strong computing capability and rich image resources.
Since wireless bandwidth from mobile devices to server
is limited, how to generate a compact signature for query to achieve low bit
data transmission becomes vital important. 2) \emph{Huge visual query variations}. Visual
splendour of a landmark can be photographed by multiple tourists under various circumstances, e.g. a wide sampling of positions, viewpoints, focal lengths, various weather conditions or illuminations. Therefore, the recorded images will have visual variations.
Besides, landmarks could be comprised of a wide range of attractive sub-spots. The images taken for sub-spots of the same landmark may appear with more visual diversity \cite{tmm/ZhuSJXZ15}. A typical example is presented in Fig. \ref{fig:diversity}. All the characteristics of landmark inevitably diversify the
visual appearances of the recorded query images, thus posing great challenges on MLS search system to robustly accommodate visual queries.
\begin{figure}
\centering
\includegraphics[width=85mm]{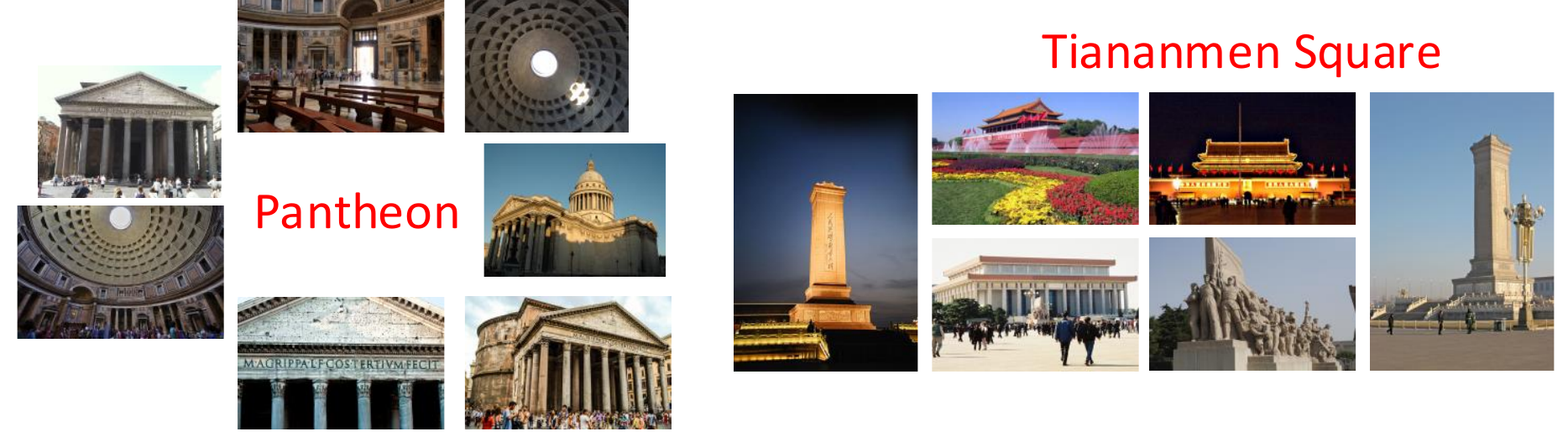}
\caption{The left sub-figure shows the images recoded for landmark \emph{Pantheon} from different viewpoints. The right sub-figure presents images recorded for different sub-pots of landmark \emph{Tiananmen Square}. All above landmark images demonstrate huge visual diversity, which presents great challenges on the mobile landmark search.}
\label{fig:diversity}
\vspace{-4mm}
\end{figure}

Hashing \cite{7915734,7047876,Song2017,leizhutcyb,leitkde2017}  is a promising technique to achieve low-bit query transmission for
supporting efficient mobile landmark search. Its main objective is to learn compact binary codes from high-dimensional data, while
preserving the similarity of the original data. Thus, it can significantly reduce transmission cost with storage-efficient binary embedding, and moreover, speedup the search process with simple but efficient Hamming distance computations \cite{DCFS,rhcmass}. However, most existing hashing strategies developed for MLS are based on uni-modal visual-words based features \cite{IJCVMLS,AIJMLS,TCFSCH}. They generally suffer from
1) limited discriminative capability and 2) poor robustness against visual variations \cite{MLSTMtao,TCFSCH}.
Although general multi-modal hashing techniques \cite{MVAGH,HASHMFTWO,HASHMFONE,MVLH} can
improve discriminative capability of binary codes \cite{tcyb/ZhuSJZX15,Xie2016}, they are designed based on simple
multiple feature integration without specific query robustness accommodation. Hence, their performance may be limited
on searching mobile landmark images.

Motivated by the above considerations, in this paper, we propose a novel hashing scheme, named as canonical view based discrete
multi-modal hashing (CV-DMH), to facilitate efficient and robust MLS. We define canonical views as the
views which capture the key characteristics of visual landmark appearances with limited redundancy. Based on them, an arbitrary image
captured by tourists or users can be robustly represented using either a specific canonical view or the cross-scenery of
multiple canonical views. Accordingly, varied visual contents of landmark can
be effectively characterized using their visual correlations to certain
canonical views. Through encoding these relations into the binary codes, various visual queries can be robustly accommodated. Furthermore, the low-cost query transmission and fast search can be well supported.

Specifically, CV-DMH works with a three-stage learning procedure.
First, a submodular function is designed to measure visual representativeness and redundancy of a view set.
With it, the optimal canonical views are efficiently discovered
by an iterative mining strategy with theoretical guarantee. Second, multi-modal sparse
coding is applied to transform visual features from multiple
modalities into a unified intermediate representation. It can robustly
characterize visual contents of varied landmark images with certain canonical views. Finally, compact binary codes are
learned on intermediate representation within a tailored discrete
binary embedding model. The learning objective is to preserve visual relations of
images measured with canonical views and remove the involved noises. In
this part, we not only explicitly deal with
the discrete constraint of hashing codes, but also consider the bit-uncorrelated
constraint and balance constraint together. A novel augmented Lagrangian
multiplier (ALM) \cite{ALMone} based optimization method is developed to directly solve
the discrete binary codes. It can desirably avoid the accumulated quantization errors in conventional
hashing method which simply adopts ``relaxing+rounding" optimization framework.

The contributions of this paper are summarized as follows:
\begin{enumerate}[1.]
  \item A submodular function is designed to measure visual representativeness
  and redundancy of a view set. With it, an iterative mining strategy is proposed
  to efficiently identify canonical views of landmarks using multiple modalities.
  Theoretical analysis demonstrates that it can obtain near-optimal solutions.

  \item A novel intermediate representation generated by multi-modal sparse coding
  is proposed to robustly characterize the visual contents of varied landmark images.
  It provides a natural and effective connection between the canonical views and binary
  embedding model.

  \item A binary embedding model tailored for canonical views is developed to
  preserve visual relations of images into binary codes and thus support efficient
  MLS with great robustness. We propose a direct discrete hashing optimization method based on ALM.
  It effectively avoids the accumulated quantization errors in
  conventional hashing methods which simply adopt ``relaxing+rounding" optimization framework.
\end{enumerate}

Compared with our previous work \cite{2CVR}, several enhancements have been made in this paper. They can be summarized as follows:
\begin{enumerate}[1.]
  \item We conduct a comprehensive review of related work and introduce more details of the proposed approach.
  \item We propose a novel discrete binary embedding model. The iterative computation for solving hashing codes enjoys an efficient optimization process. Moreover, direct discrete optimization can successfully avoid the accumulated quantization errors.
  \item More experiments are conducted and the presented results validate the effects of the proposed approach.
\end{enumerate}

The rest of the study is structured as follows. Section \ref{sec:2} introduces related work.
System overview of CV-DMH based MLS system is illustrated in Section \ref{sec:so}.
Details about CV-DMH are introduced in Section \ref{sec:3}. Experimental configuration is presented
in Section \ref{sec:4}. In Section \ref{sec:5}, we give experimental results and analysis.
Section \ref{sec:6} concludes the study with a summary and future work.

\section{Related Work}
\label{sec:2}
Due to the constraints of space, only the work most related to this study is introduced in this section. In particular,
we present a short literature review on mobile landmark search, multi-modal hashing, and discrete hashing optimization.

\subsection{Mobile Landmark Search}
Most existing approaches developed for mobile landmark search focus on compressing high-dimensional visual-words based landmark feature into binary codes to achieve low bit rate data transmission. Ji \emph{et al}. \cite{IJCVMLS} present a location discriminative vocabulary coding (LDVC) to compress bag-of-visual-words (BoVW) with location awareness by combining both visual content and geographical context. The compact landmark descriptor is generated with an iterative optimization between geographical segmentation and descriptor learning. Duan \emph{et al}. \cite{AIJMLS} explore multiple information sources, such as landmark image, GPS, crowd-sourced hotspot
WLAN, and cell tower locations, to extract location discriminative compact landmark image
descriptor. Chen \emph{et al}. \cite{MLSTMtao} develop a soft bag-of-visual phrase (BoVP)
to learn category-dependent visual phrases, by capturing co-occurrence features of
neighbouring visual-words.  The context location and direction information captured by
mobile devices are also integrated with the proposed BoVP. To alleviate the online memory cost,
Zhou \emph{et al}. \cite{TCFSCH} propose codebook-free scalable cascaded hashing (SCH) for mobile landmark search.
In their approach, matching recall rate is ensured first and false positive matches are then removed with a subsequent verification step.

All the aforementioned techniques learn compact binary codes from only visual-words based
features, without considering valuable information from other visual modalities. Therefore,
they will generate binary codes with limited discriminative information. To the best of our knowledge,
only \cite{MLS3D} is proposed to robustly accommodate landmark images under various conditions. However, it
sends compressed query images with low resolution for landmark search, Consequently, it still consumes considerable transmission bandwidth.
\begin{figure*}
\centering
\includegraphics[width=180mm]{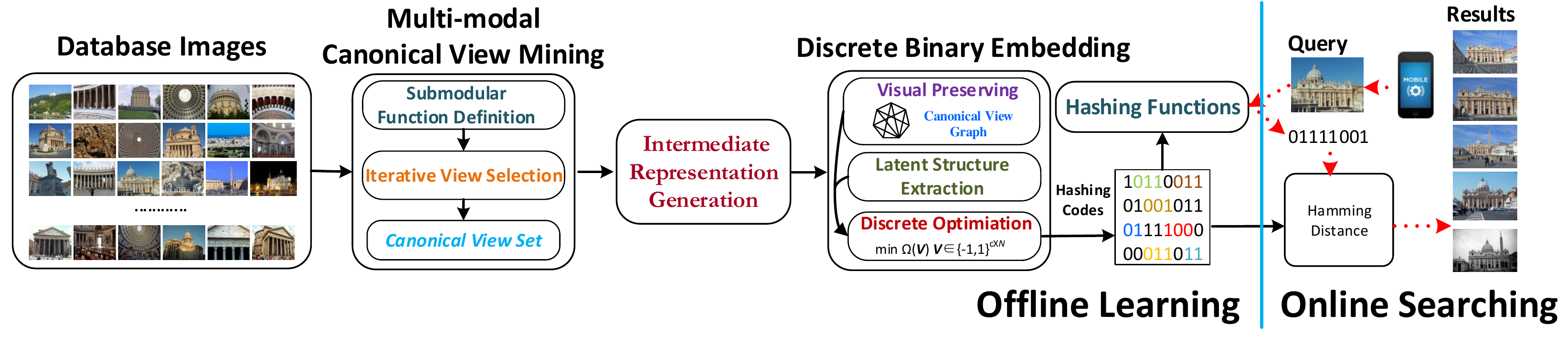}
\vspace{-3mm}
\caption{Framework of the CV-DMH based mobile landmark search system. The main aim of the offline learning is
to learn hashing codes of database images and hashing functions for online query. This part mainly consists of three steps: multi-modal canonical view mining, intermediate representation generation, and discrete binary embedding. The online searching part generates hashing codes for queries and performs efficient online similarity search in Hamming space.}
\label{fig:sys}
\vspace{-3mm}
\end{figure*}

\subsection{Multi-modal Hashing}
Multi-modal hashing has been emerging as a promising technique to generate compact binary
code based on multiple features. The earliest study on this topic is composite hashing with
multiple information sources (CHMIS) \cite{HASHMFFOUR}. It integrates discriminative information from multiple
sources into the hashing codes via weight adjustment on each individual source to maximize the hashing embedding performance.
However, CHMIS just post-integrates linear output of features and fails to fully exploit their correlations.
Kim \emph{et al}. \cite{MVAGH} present multi-view anchor graph hashing (MVAGH) to extend anchor graph hashing
(AGH) \cite{HASHGRAPH} in multiple views (visual modalities). The hashing codes are determined as the subset
of eigenvectors calculated from an averaged similarity matrix induced by multiple anchor graphs.
Song \emph{et al}. \cite{HASHMFTWO} develop multiple feature hashing (MFH). The learning process preserves
the local structure information of each individual feature, and simultaneously considers the global
alignment of local structures for all the features. By using the learned hashing hyper-plane,
MFH first concatenates all the features into a single vector and then projects it into binary codes.
Liu \emph{et al}. \cite{HASHMFONE} propose compact kernel hashing (CKH). It formulates
the similarity preserving hashing with linearly combined multiple kernels corresponding to different features.
More recently, multi-view latent hashing (MVLH) \cite{MVLH} is proposed to incorporate
multi-modal features in binary hashing learning by discovering the latent factors
shared by multiple modalities.  Multi-view alignment hashing (MVAH) is presented in \cite{TIPMVAH} to combine
multiple image features for learning effective hashing functions based on the regularized kernel nonnegative matrix factorization.

Distinguished from the hashing methods presented above, CV-DMH learns multi-modal hashing codes on canonical views
by capturing the key characteristics of landmarks. With this design, the generated binary codes
can enjoy desirable robustness on query accommodation. Moreover, CV-DMH directly solves discrete hashing codes. It can effectively avoid the accumulated quantization errors in existing multi-modal hashing methods adopting simplified ``relaxing+rounding" optimization framework.

\subsection{Discrete Hashing Learning}
Most existing hashing methods simply apply two-step ``relaxing+rounding" optimization framework to solve hashing codes. However, as indicated by recent literature \cite{DGH,SDH,luo2017robust}, this simple relaxing will bring significant information quantization loss. To alleviate this problem, several approaches are proposed to directly deal with the discrete optimization challenge. Discrete graph hashing (DGH) \cite{DGH} reformulates the graph hashing with a discrete optimization framework and solves the problem with a tractable alternating maximization algorithm. Supervised discrete hashing (SDH) \cite{SDH} learns discrete hashing codes via supervised learning. Cyclic coordinate descent is applied to calculate discrete hashing bits in a closed form. Coordinate discrete hashing (CDH) \cite{CDOECVR} is designed for cross-modal hashing \cite{xu2017learning}, and the discrete optimization proceeds in a block coordinate descent manner. Column sampling based discrete supervised hashing (COSDISH) is proposed in \cite{CSBDSH} to learn discrete hashing codes from semantic information by column sampling. Discrete proximal linearized minimization (DPLM) is presented in \cite{TIP2016binary} to solve discrete hashing codes. It reformulates the hashing learning as minimizing the sum of a smooth loss term with a nonsmooth indicator function. The problem is finally solved by an iterative procedure with each iteration admitting an analytical discrete solution. Kernel-based supervised discrete hashing (KSDH) \cite{Shi2016} solves discrete hashing codes via asymmetric relaxation strategy. It relaxes the hashing function into a general binary code matrix which is calculated within an alternative strategy. Although these approaches achieve certain success, their proposed discrete optimization solutions are specially designed for uni-modal hashing and particular hashing types. Hence, they cannot be easily generalized to other hashing learning formulations.

Our work is an advocate of discrete hashing optimization but focuses on learning robust hashing codes for mobile landmark search. We integrate the discriminative information from multi-modal features directly into discrete hashing codes based on the discovered informative canonical views. Moreover, our proposed discrete optimization strategy can not only explicitly deal with the discrete constraint of binary codes, but also consider the bit-uncorrelated constraint and balance constraint together\footnote{DPLM can cope with uncorrelation and balance constraints. However, it simply transfers two constraints to the objective function and tries to optimize a relaxed equivalent problem.}. To the best of our knowledge, there is still no similar work.

\section{System Overview}
\label{sec:so}
This section briefly introduces system overview of the proposed CV-DMH based mobile landmark search system. As shown in the Fig. \ref{fig:sys}, the system is mainly comprised of two key components: offline learning and online searching.
\begin{itemize}
\itemsep=3pt
\item \emph{Offline Learning}: The aim of this part is to learn hashing functions which can project high-dimensional multi-modal features of both query and database images into binary hashing codes. Specifically, offline learning is further divided into four subsequent sub-processes: feature extraction, multi-modal canonical view mining, intermediate representation generation, and discrete binary embedding. In the system, multi-modal features are first extracted from heterogeneous visual modalities to represent images. Then, multi-modal canonical view mining is proposed in order to efficiently discover a compact but informative canonical view set from noisy landmark image collections to capture key visual appearances of landmarks. Next, in order to robustly model diverse visual contents, an intermediate representation is generated by computing multi-modal sparse reconstruction coefficients between image and canonical views. Finally, compact binary codes of database images and hashing functions for online queries are learned by formulating a unified discrete binary embedding model.

\item \emph{Online Searching}: Query image is first submitted by user from mobile devices. Multi-modal visual features of it are then extracted on image and transformed to an intermediate image representation with the same pipeline conducted on database images. Next, binary codes of query are generated with the hashing functions learned from offline learning. Finally, the Hamming distances are computed with simple bit operations and ranked in ascending order, and their corresponding landmark images are returned.
\end{itemize}

\section{The Proposed CV-DMH}
\label{sec:3}
This section provides the details of the proposed CV-DMH. First, we introduce multi-modal canonical view mining. Second, we give details of intermediate representation generation. Third, we formulate the discrete binary embedding model and give an efficient discrete optimization solution. Finally, we summarize the key steps of CV-DMH and give a computation complexity analysis.

\subsection{Multi-modal Canonical View Mining}
\label{sec:3:2}

Landmark images have an interesting characteristic that is different from general images. Landmark is located on specific venue or attraction. On the one hand, in real practice, only the spectacular and attractive views of a landmark will be photographed by various tourists spontaneously. On the other hand, many of them would like to share their recorded images in social websites (e.g. Flickr). These accumulated canonical views of landmarks in social websites coincidentally reflect the common preferences of tourists. On the perspective of technique, they capture the key visual characteristics of landmarks, and an arbitrary query image recorded by tourists can be characterized with several canonical views of landmark. Therefore, it is promising to leverage canonical views of landmark as hashing learning bases to robustly accommodate visual variations of query images captured by mobile devices. Principally, optimal canonical views should possess two important properties: 1) Representative. The canonical views should capture the key visual landmark components from large quantities of recorded landmark images. 2) Compact. Redundant views will bring noises and increase extra computation burden.

Motivated by the above analysis, we propose an efficient submodular function based mining algorithm, measuring the representativeness and compactness of view set, to iteratively discover canonical views. Specifically, we first quantitatively define the above two properties of canonical view set\footnote{Canonical view set is comprised of canonical views.}. Then, a submodular function is designed accordingly and an efficient iterative mining approach with theoretical guarantee is developed for canonical view discovery.
\begin{definition}
Let $\mathcal I$ denote image space, $\mathcal I= \{\mathcal I_n\}_{n=1}^N$, $N$ is the number of database images.
Let $\mathcal L$ denote landmark space, $\mathcal L= \{\mathcal L_m\}_{m=1}^M$, $M$ is the number of landmarks in database.
$\mathcal L_m$ is defined as a set of images which are recorded at the nearby positions of the $m_{th}$ landmark.
Let $\mathcal V$ denote a view set of $\mathcal L$. It is defined as a set of images $\{v_i\}_{i=1}^{|\mathcal V|}$ belonging to $\mathcal I$, $\mathcal V \subseteq \mathcal I$, $|\mathcal V|\ll |\mathcal I|$.
\end{definition}

\begin{definition}
Let $\texttt{Rep}(\mathcal V)$
denote the visual representativeness of view set $\mathcal V$ over $\mathcal L$.
It is defined as $\texttt{Rep}(\mathcal V)=\sum_{v_i\in \mathcal V}\texttt{Rep}(v_i)=\sum_{v_i\in \mathcal V} \sum_{v_j\in \mathcal I, i\neq j}g_{ij}$, $g$ is the function which measures the
feature similarity of two image views, $g_{ij}$ is short for $g(v_i,v_j)$.
Let $\texttt{Red}(\mathcal V)$ denote the visual redundancy of view set $\mathcal V$.
It is defined as $\texttt{Red}(\mathcal V)=\sum_{v_i\in V}
\texttt{Red}(v_i)=\sum_{v_i,v_j\in \mathcal V, i\neq j} g_{ij}$.
\end{definition}

\begin{definition}
Let $\mathcal C$ denote the canonical view set of $\mathcal L$.
The views involved in $\mathcal C$ can comprehensively represent
diverse visual contents of landmark, and meanwhile, have less visual redundancy. In this paper, it is defined as
$\mathcal C=\arg\max_{\mathcal V\subseteq \mathcal I, |\mathcal V|=T} h(\mathcal V), h(\mathcal V)=\texttt{Rep}(\mathcal V)- \texttt{Red}(\mathcal V)$, $T$ is cardinality of canonical view set.
\end{definition}
\begin{lemma}
$h(\mathcal V)$ is submodular function. That is, $\forall \mathcal V_1\subseteq \mathcal V_2\subseteq \mathcal V, \forall v_j\notin \mathcal V$, $h(\mathcal V_1\cup v_j)-h(\mathcal V_1)\geq h(\mathcal V_2\cup v_j)-h(\mathcal V_2)$.
\end{lemma}
\vspace{-4mm}
\begin{equation*}
\begin{aligned}
\small
& \emph{Proof} \quad \texttt{Rep}(\mathcal V_1\cup v_j)-\texttt{Rep}(\mathcal V_1) =\texttt{Rep}(\mathcal V_2\cup v_j)- \\
&\texttt{Rep}(\mathcal V_2)-(\texttt{Red}(\mathcal V_1\cup v_j)-\texttt{Red}(\mathcal V_1))=-2\sum_{v_i\in \mathcal V_1 \setminus v_j} g_{ij}\\
&\geq -2\sum_{v_i\in \mathcal V_2 \setminus v_j} g_{ij} = -(\texttt{Red}(\mathcal V_2\cup v_j)-\texttt{Red}(\mathcal V_2)) \\
&\Rightarrow \ h(\mathcal V_1\cup v_j)-h(\mathcal V_1)\geq h(\mathcal V_2\cup v_j)-h(\mathcal V_2)
\end{aligned}
\end{equation*}
\begin{lemma}
$h(\mathcal V)$ is monotonically nondecreasing function. That is, $\forall \mathcal V_1\subseteq \mathcal V_2\subseteq \mathcal V$, $h(\mathcal V_1)\leq h(\mathcal V_2)$.
\end{lemma}
\vspace{-4mm}
\begin{equation*}
\small
\begin{aligned}
&\emph{Proof} \quad h(\mathcal V)=\sum_{v_i\in \mathcal V}\sum_{v_j\in \mathcal I \setminus \mathcal V, i\neq j} g_{ij} \\
& \Rightarrow \ h(\mathcal V_1)=\sum_{v_i\in \mathcal V_1}\sum_{v_j\in \mathcal V_2 \setminus \mathcal V_1,i\neq j} g_{ij}+ \sum_{v_i\in \mathcal V_1}\sum_{v_j\in \mathcal I \setminus \mathcal V_2,i\neq j} g_{ij}\\
& \Rightarrow \ h(\mathcal V_2)=\sum_{v_i\in \mathcal V_2\setminus\mathcal V_1} \sum_{v_j\in \mathcal I \mathcal \setminus V_2, i\neq j} g_{ij} + \sum_{v_i\in \mathcal V_1}\sum_{v_j\in \mathcal I \setminus \mathcal V_2,i\neq j} g_{ij}\\
&Since \ in \ our \ case, \ |\mathcal V_1|, |\mathcal V_2| \ll \mathcal I \\
& \Rightarrow \ \sum_{v_i\in \mathcal V_1}\sum_{v_j\in \mathcal V_2 \setminus \mathcal V_1,i\neq j} g_{ij}\leq \sum_{v_i\in \mathcal V_2\setminus\mathcal V_1} \sum_{v_j\in \mathcal I \setminus \mathcal  V_2, i\neq j} g_{ij}\\
& \Rightarrow \ h(\mathcal V_1)\leq h(\mathcal V_2)
\end{aligned}
\end{equation*}

As indicated in \textbf{Definition 3}, to discover optimal canonical
view set, the function $h(\mathcal V)$ should be maximized in theory. However,
since $h(\mathcal V)$ is submodular function, the maximization of it
is a NP-complete optimization problem. Fortunately, $h(\mathcal V)$ is monotonically nondecreasing
with a cardinality constraint as indicated by two above Lemmas. Canonical views can be discovered near optimally
by greedy strategy as following steps\vspace{-3mm}\newline

\begin{algorithm}
\caption{Canonical view mining}
\label{cvm}
\begin{algorithmic}[1]
\STATE Extract visual feature for all landmark images. \\
\STATE Set canonical view set as empty, $\mathcal C=\emptyset$.\\
\FOR{$t=1$ to $T$}
\STATE Compute $\texttt{diff}(\mathcal I_{n})=h(\mathcal C \cup \mathcal I_{n})-h(\mathcal I_{n})$
for each landmark image $\mathcal I_{n}\in \mathcal I$. \\
\STATE Select the image with the maximum $\texttt{diff}$ into $\mathcal C$, $\mathcal I^*=\arg\max_{I_n\in \mathcal I}\texttt{diff}(I_n)$, and
simultaneously remove it from $\mathcal I$, $\mathcal C \leftarrow \mathcal C \cup \mathcal I^*$, $\mathcal I \leftarrow \mathcal I \setminus \mathcal I^*$.\\
\ENDFOR
\end{algorithmic}
\end{algorithm}
\begin{theorem}
\cite{AAAItheorem} Let $\mathcal S^*$ denote the global optimal solution that solves the combinatorial optimization problem $\arg\max_{\mathcal S\subseteq \mathcal L, |\mathcal S|=T} h(\mathcal S)$, $\mathcal S$ denote approximate solution found by the greedy algorithm. If $h(\mathcal S)$ is nondecreasing
submodular function with $h(\emptyset)=0$, we can have
\begin{equation*}
\begin{aligned}
\small
h(\mathcal S)\geq h(\mathcal S^*)\frac{\zeta-1}{\zeta}
\end{aligned}
\end{equation*}
where $\zeta$ refers to the natural exponential.
\end{theorem}
As validated by \textbf{Theorem 1}, this greedy algorithm can achieve a result that is no worse than a constant
fraction $\frac{\zeta-1}{\zeta}$ away from the optimal value. The time complexity of canonical view mining is reduced to $O(NT)$.
Hence, the canonical view discovery process can be completed efficiently. Different modalities generally include complete information \cite{7565615,DBLP:conf/cvpr/ChangYYX16,yang2014exploiting,Cheng2016,ChangMLYH17,7442559}. Canonical views in different modalities may be different.
To comprehensively cover visual appearances of landmarks,
canonical view mining is performed in multiple modalities,
obtaining canonical view set $\{\mathcal C^p\}_{p=1}^P$, $P$ is number of modalities.
We concatenate features of canonical views and
construct a matrix $E^p=[e_1^p, ..., e_{T}^p]\in \mathbb{R}^{d_p\times T}$
in modality $p$, $d_p$ is the corresponding feature dimension.
\begin{figure}
\centering
\includegraphics[width=85mm]{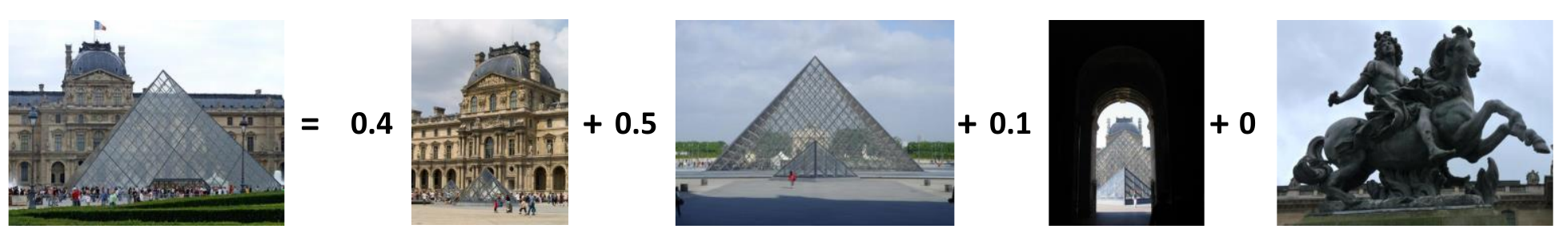}
\caption{Visual contents of an arbitrary query landmark image can be sparsely represented with its relations to canonical views.}
\label{fig:sample}
\vspace{-4mm}
\end{figure}
\subsection{Intermediate Representation Generation}
\label{sec:3:3}
With the discovered canonical views, we generate an intermediate representation for subsequent binary embedding.
As illustrated above, an arbitrary recorded landmark image
either describes a certain canonical view
or the cross-scenery among several particular canonical views.
In both cases, visual contents of the image can be sparsely represented with
its relations to several particular canonical views (as shown in Fig. \ref{fig:sample}). Sparse coding performs well on
robust representation with fix visual bases. In this paper, we leverage it for intermediate representation generation. Specifically, we calculate multi-modal sparse reconstruction coefficients
between image and canonical views, and the auto-generated response coefficients
are determined as the dimensions of intermediate representation. Principally, the intermediate representation can effectively characterize diverse visual
contents by adaptively adjusting the response coefficients on canonical views. Hence, it can construct a robust
foundation for subsequent discrete binary embedding. Mathematically, the concrete computation form is
\begin{equation}
\small
\begin{aligned}
\label{eq:afm}
\min_{\{Y^p\}_{p=1}^P} \ & \sum_{p=1}^P ||X^p-E^pY^p||_F^2 + \sigma\sum_{p=1}^P\sum_{n=1}^N||d_n^p\otimes y_n^p||_F^2 \\
s.t. \ & 1_T^\texttt{T}y_n^p=1, d_n^p=\texttt{exp}(\frac{\texttt{dist}(x_n^p, E^p)}{\rho}), \forall p, n \\
\end{aligned}
\end{equation}

\noindent where the first term measures the sparse reconstruction errors from raw features to exemplars, the second term is locality adaptor that ensures the generated sparse codes are proportional to the similarities between their corresponding raw features and exemplars. $\sigma>0$ is a constant factor that adjusts the balance between terms, $\rho$ is set to be the mean of pairwise distances, $\otimes$ denotes the element-wise product, $1_T\in \mathbb{R}^\texttt{T}$ denotes a column vector with all ones. $X^p=[x_1^p,..., x_N^p]\in \mathbb{R}^{d_p\times N}$ denotes
features of database images in modality $p$. $Y^p=[y_1^p,...,y_N^p]\in \mathbb{R}^{T\times N}$
denotes modality-specific canonical view based intermediate representation. Each column
has $r$ non-zeros coding coefficients.
$d_n^p$ is a vector that measures the distances between raw feature and exemplars,
$\texttt{dist}(x_n^p, E^p)= [\texttt{dist}(x_n^p, e_1^p), ..., \texttt{dist}(x_n^p, e_{T}^p)]$, $\texttt{dist}(x_n^p, e_1^p)$ is Euclidean distance between
$x_n^p$ and $e_{T}^p$.
The problem in (\ref{eq:afm}) can be efficiently solved by using the alternating direction method of multipliers (ADMM) \cite{PAMISSC}. After solving it, we concatenate the calculated $Y^p$ and construct dimensions of intermediate representation
\begin{equation}
\label{calculatey}
\small
\begin{aligned}
Y=[Y^1; ... ; Y^P]\in \mathbb{R}^{TP\times N}
\end{aligned}
\end{equation}


\subsection{Discrete Binary Embedding Model}
\label{sec:3:4}
Based on the intermediate representation, we design a discrete binary embedding model to
learn final binary hashing codes. Let us define $V\in [-1, 1]^{c\times N}$ as the hashing codes of database images,
$c$ is the hashing code length. Due to approximate canonical view mining, the intermediate representation
generation inevitably brings about noises from inaccurate canonical views.
In addition, as a result of information integration
from multiple modalities, there exist information redundancies among dimensions of
intermediate representation calculated from different modalities. Therefore, it is very important to remove the involved noises and redundancies
during the binary embedding. To achieve this goal, we propose latent structure learning
with matrix factorization to extract orthogonal Hamming space. In this space, each dimension corresponds to one hashing bit. Besides, we construct a canonical view based graph \cite{7547296,JGLVS,GWC} to measure visual relationships among images captured by canonical views. That is, if two landmark images have similar visual distributions on intermediate representations, they are forced to be projected to close points in hamming space. Moreover, we learn linear projection based hashing functions to support queries that are out of the database.

By integrating the aforementioned considerations, we derive the overall discrete binary embedding formulation as
\begin{equation}
\small
\label{disobj}
\begin{aligned}
\min_{W, U, V} \ & ||Y-UV||_F^2 + \alpha Tr(VLV^{\texttt{T}}) + \beta(||V- W^{\texttt{T}}Y||_F^2+\gamma ||W||_F^2) \\
s.t. \ & V\in \{-1,1\}^{c\times N}, VV^{\texttt{T}}=NI_c, V1_N=0
\end{aligned}
\end{equation}

\noindent where $\alpha, \beta, \gamma>0$ adjust the balance of terms. $V\in [-1, 1]^{c\times N}$ is discrete constraint on hashing codes.
In addition, we also consider the bit-uncorrelated constraint $VV^{\texttt{T}}=NI_c$ which
guarantees that the learned Hamming space to be orthogonal for information redundancy removal, and bit balance constraint $V1_N=0$
which forces each bit to have equal chance to occur in the whole image database. $||Y-UV||_F^2$
is for latent structure learning, $U$ is mapping between latent structure and intermediate representation.
$Tr(VLV^{\texttt{T}})$ is graph regularizer which
preserves visual relations of landmark images measured by canonical views.
$Tr(\cdot)$ is trace operation, $L$ is Laplacian matrix calculated on canonical view based graph, it is
constructed based on intermediate representation, to measure the visual relations of images
on canonical views. $||V-W^{\texttt{T}}Y||_F^2+\gamma ||W||_F^2$ learns
linear projection matrix $W\in \mathbb{R}^{TP\times c}$ equipped in hashing functions,
$||V-W^{\texttt{T}}Y||_F^2$ is to reduce the loss between
binary codes and the projected values. It is worth noting that,
as linear projection is leveraged, the online mobile landmark search process is efficient.


The optimal $W$ that solves Eq.(\ref{disobj}) can be expressed in terms of $Y$. We can derive the following theorem.
\begin{theorem}
Let $W, U, V$ and $Y$ be defined as before. Then the optimal $W$ that solves learning problem in (\ref{disobj}) is given by
$W=(YY^{\texttt{T}}+\gamma I)^{-1}YV^{\texttt{T}}$.
The minimum problem in (\ref{disobj}) is equivalent to the following problem
\begin{equation}
\small
\label{obj:1}
\begin{aligned}
\min_{U, V} \ & ||Y-UV||_F^2+\alpha Tr(VAV^{\texttt{T}}) \\
s.t. \ & V\in \{-1,1\}^{c\times N}, VV^{\texttt{T}}=NI_c, V1_N=0
\end{aligned}
\end{equation}
\noindent where $A=L+\frac{\beta}{\alpha}(I-Y^{\texttt{T}}QY), Q=(YY^{\texttt{T}}+\gamma I)^{-1}$.
\end{theorem}
\begin{proof}


By calculating the derivation of
the objective function in Eq.(\ref{disobj}) \emph{w.r.t} $W$ and setting it to be zero, we can have
\begin{equation}
\small
\label{eq:3:1}
\begin{aligned}
W=(YY^{\texttt{T}}+\gamma I)^{-1}YV^{\texttt{T}} \\
\end{aligned}
\end{equation}
Let $Q=(YY^{\texttt{T}}+\gamma I)^{-1}$, then $W=QYV^{\texttt{T}}$.
By replacing $W$ into Eq.(\ref{disobj}), we can derive that
\begin{equation*}
\small
\label{formuM}
\begin{aligned}
||V-W^{\texttt{T}}Y||_F^2+\gamma ||W||_F^2= Tr(V(I_N-Y^{\texttt{T}}QY)V^{\texttt{T}})\\
\end{aligned}
\end{equation*}

By summing three terms together, we derive that
\begin{equation*}
\small
\begin{aligned}
&\ ||Y-UV||_F^2 + \alpha Tr(VLV^{\texttt{T}})+ \beta Tr(V(I-Y^{\texttt{T}} \\
&(YY^{\texttt{T}}+\gamma I)^{-1}Y)V^{\texttt{T}}) = ||Y-UV||_F^2+\alpha Tr(VAV^{\texttt{T}})\\
\end{aligned}
\end{equation*}
\noindent where $A=L+\frac{\beta}{\alpha}(I-Y^{\texttt{T}}QY), Q=(YY^{\texttt{T}}+\gamma I)^{-1}$. This completes the proof of the theorem.
\end{proof}

With \textbf{Theorem 2}, the problem in (\ref{disobj}) is transformed to
\begin{equation}
\small
\begin{aligned}
\label{disretobject}
&\min_{U, V} ||Y-UV||_F^2+\alpha Tr(VAV^{\texttt{T}})\\
s.t. \ & V\in \{-1,1\}^{c\times N}, VV^{\texttt{T}}=NI_c, V1_N=0
\end{aligned}
\end{equation}

The above objective function can be abstracted as graph-based matrix factorization \cite{TPAMINMF}. The main difference between our formulation and \cite{TPAMINMF} is the constraints imposed on hashing codes. These three constraints bring new challenges to the optimization process which have not yet been touched by existing graph-based matrix factorization methods.
\subsection{Discrete Solution}
\label{sec:3:5}
Essentially, solving problem (\ref{disretobject}) is  a challenging combinatorial optimization problem due to three constraints. Most existing hashing approaches apply ``relaxing+rounding" optimization \cite{hashingsurvey}. They first relax three constraints to calculate continuous values, and then binarize them to hashing codes via rounding. This two-step learning can actually simplify the solving process, but it may cause significant information loss \cite{SDH,DGH}. In recent literature, several discrete hashing solutions are proposed. However, they are developed for particular hashing types. For example, graph hashing \cite{DGH}, supervised hashing \cite{SDH,CSBDSH,KSDH}, cross-modal hashing \cite{CDOECVR}. Therefore, their designed learning strategies cannot be directly applied to solve our problem.

In this paper, we propose an effective optimization algorithm based on augmented Lagrangian multiplier (ALM) \cite{ALMone} to calculate the discrete solution within one step. Our method can not only explicitly deal with the discrete constraint, but also consider the bit-uncorrelated constraint and balance constraint together. Our key idea is adding auxiliary variables $\Gamma$, $\Theta$ to separate three challenging constraints, and transform the objective function to an equivalent one. The substituted variables and corresponding auxiliary ones are forced to be close with each other. Specifically, in this paper, we set $\Gamma=Y-UV, \Theta=V$. Problem (\ref{disretobject}) is transformed as
\begin{equation}
\small
\begin{aligned}
\label{eq:tj}
& \min_{U, V, \Gamma, \Theta} \ ||\Gamma||_F+ \frac{\eta}{2} ||Y-UV-\Gamma+\frac{E_{\eta}}{\eta}||_F^2 \\
& \alpha Tr(VA\Theta^{\texttt{T}}) + \frac{\mu}{2} ||V-\Theta+\frac{E_{\mu}}{\mu}||_F^2 \\
& s.t. \quad V\in \{-1,1\}^{c\times N}, \Theta\Theta^\texttt{T}=NI_c, \Theta1_N=0
\end{aligned}
\end{equation}
\noindent where $E_{\eta}, E_{\mu}$ measure the difference between the target and auxiliary variables, $\eta, \mu>0$ adjust the balance between terms.
We adopt alternate optimization to iteratively solve problem (\ref{eq:tj}). We optimize the objective function with respective to one variable while fixing other remaining variables. The key steps for solving $Z$ are summarized in Algorithm \ref{calv}.
\begin{algorithm}
\caption{Obtaining $V$ via solving problem (\ref{eq:tj})}
\label{calv}
\begin{algorithmic}[1]
\STATE Initialize $E_{\eta}$, $E_{\mu}$, $V$; \\
\WHILE{not convergence}
\STATE Optimize $\Gamma$ while fixing the others with Eq.(\ref{eq:tl});
\STATE Optimize $U$ while fixing the others with Eq.(\ref{eq:tp});
\STATE Optimize $\Theta$ while fixing the others with Eq.(\ref{eq:optitheta});
\STATE Optimize $V$ while fixing the others with Eq.(\ref{eq:tz});
\STATE Update $E_{\eta}$, $E_{\mu}$, $\eta$, $\mu$ with Eq.(\ref{eq:tu}); \\
\ENDWHILE
\end{algorithmic}
\end{algorithm}

\textbf{Step 3: Update $\Gamma$.} By fixing other variables, the optimization formulas for $\Gamma$ are
\begin{equation}
\small
\begin{aligned}
\label{eq:tk}
&\min_{\Gamma} \ ||\Gamma||_F+ \frac{\eta}{2} ||Y-UV-\Gamma+\frac{E_{\eta}}{\eta}||_F^2)
\end{aligned}
\end{equation}
\noindent By calculating the derivative of the objective function with respective to $\Gamma$, and setting it to 0, we can obtain that
\begin{equation}
\small
\begin{aligned}
\label{eq:tl}
\Gamma=\frac{\eta Y-\eta UV+E_{\eta}}{2+\eta}
\end{aligned}
\end{equation}

\textbf{Step 4: Update $U$}. The optimization formula for $U$ is
\begin{equation}
\small
\begin{aligned}
\label{eq:tn}
\min_{U} \ & ||Y-UV-\Gamma + \frac{E_{\eta}}{\eta}||_F^2 \\
\end{aligned}
\end{equation}
\noindent By calculating the derivative of the objective function with respective to $U$, and setting it to 0, we can obtain that
\begin{equation}
\small
\begin{aligned}
UV = Y-\Gamma+\frac{E_{\eta}}{\eta} \\
\end{aligned}
\end{equation}

\noindent Since $VV^\texttt{T}=NI_c$, we can further derive that
\begin{equation}
\small
\begin{aligned}
\label{eq:tp}
U = \frac{1}{N}(Y-\Gamma+\frac{E_{\eta}}{\eta})V^\texttt{T}
\end{aligned}
\end{equation}

\textbf{Step 5: Update $\Theta$}. The optimization formula for $\Theta$ is
\begin{equation}
\small
\begin{aligned}
\label{eq:tq}
\min_{\Theta} \ & \alpha Tr(VA\Theta^\texttt{T})+\frac{\mu}{2}||V-\Theta+\frac{E_{\mu}}{\mu}||_F^2 \\
& s.t. \ \Theta\Theta^\texttt{T}=NI_c, \Theta1_N=0
\end{aligned}
\end{equation}
\noindent The objective function in Eq.(\ref{eq:tq}) can be simplified as
\begin{equation}
\small
\begin{aligned}
\label{eq:tre}
\min_{\Theta} & \ ||\Theta-(V+\frac{E_{\mu}}{\mu}-\frac{\alpha}{\mu}VA)||_F^2\\
\end{aligned}
\end{equation}
\noindent where $C=V+\frac{E_{\mu}}{\mu}-\frac{\alpha}{\mu}VA$.
The above equation is equivalent to the following maximization problem
\begin{equation}
\small
\begin{aligned}
\label{eq:ts}
\max_{\Theta} \ & Tr(\Theta^\texttt{T}C) \quad s.t. \ \Theta\Theta^\texttt{T}=NI_c, \Theta1_N=0
\end{aligned}
\end{equation}
Mathematically, with singular value decomposition (SVD), $C$ can be decomposed as $C=P\Lambda Q^\texttt{T}$
, where the columns of $P$ and $Q$ are left-singular vectors and right-singular vectors of $C$ respectively, $\Lambda$ is rectangular diagonal matrix and its diagonal entries are singular values of $C$. Then, we can derive that $\max_{V} \ Tr(\Theta^\texttt{T}P\Lambda Q^\texttt{T}) \Leftrightarrow \max_{V} \ Tr(\Lambda Q^\texttt{T}\Theta^\texttt{T}P)$.
\begin{theorem}\label{calzthre}
Given any matrix $G$ which meets $GG^\texttt{T}=NI$ and diagonal matrix $\Lambda \ge 0$, the solution of $\max_{G} Tr(\Lambda G)$ is $\texttt{diag}(\sqrt{N})$.
\end{theorem}
\begin{proof}
Let us assume $\lambda_{ii}$ and $g_{ii}$ are the $i_{th}$ diagonal entry of $\Lambda$ and $G$ respectively, $Tr(\Lambda G)=\sum_{i}\lambda_{ii}g_{ii}$. Since $GG^\texttt{T}=NI$, $g_{ii}\leq \sqrt{N}$. $Tr(\Lambda G)=\sum_{i}\lambda_{ii}g_{ii}\leq \sqrt{N}\sum_{i}\lambda_{ii}$. The equality holds only when $g_{ii}=\sqrt{N}, g_{ij}=0, \forall i,j$. $Tr(\Lambda G)$ achieves its maximum when $G=\texttt{diag}(\sqrt{N})$.
\end{proof}

As $\Lambda$ is calculated by SVD, we can obtain that $\Lambda \ge 0$. On other hand, we can easily derive that $Q^\texttt{T}\Theta^\texttt{T}PP^\texttt{T}\Theta Q=NI$. Therefore, according to the \textbf{Theorem \ref{calzthre}}, the optimal $\Theta$ can only be obtained when $Q^\texttt{T}\Theta^\texttt{T}P=\texttt{diag}(\sqrt{N})$. Hence, the solution of $\Theta$ is
\begin{equation}
\small
\begin{aligned}
\label{eq:solutiony}
\Theta=\sqrt{N}PQ^\texttt{T}
\end{aligned}
\end{equation}
Moreover, in order to satisfy the balance constraint $\Theta1_N=0$, we apply Gram-Schmidt process as \cite{DGH} to construct matrices $\hat{P}$ and $\hat{Q}$, so that $\hat{P}^\texttt{T}\hat{P}=I_{L-R}$, $[P, 1]^\texttt{T}\hat{P}=0$, $\hat{Q}^\texttt{T}\hat{Q}=I_{L-R}$, $Q\hat{Q}^\texttt{T}=0$, $R$ is the rank of $C$. The close form solution for $\Theta$ is
\begin{equation}
\small
\begin{aligned}
\label{eq:optitheta}
\Theta=\sqrt{N}[P, \hat{P}][Q, \hat{Q}]^\texttt{T}
\end{aligned}
\end{equation}

\textbf{Step 6: Update $V$}. By fixing other variables, the optimization formula for $V$ is
\begin{equation}
\small
\begin{aligned}
\label{eq:tr}
\min_{V\in [-1, 1]^{c\times N}} \ & \frac{\eta}{2}||Y-UV-\Gamma+\frac{E_{\eta}}{\eta}||_F^2 +\alpha Tr(VA\Theta^{\texttt{T}})\\
& +\frac{\mu}{2} ||V-\Theta+\frac{E_{\mu}}{\mu}||_F^2 
\end{aligned}
\end{equation}
\noindent The above problem can be transformed as
\begin{equation}
\small
\begin{aligned}
\label{eq:trrr}
\min_{V\in [-1, 1]^{c\times N}} &  \ ||V-(\Theta-\frac{E_{\mu}}{\mu}-\frac{\alpha}{\mu}\Theta A + \frac{\eta}{\mu} U^\texttt{T}(Y-\Gamma+\frac{E_{\eta}}{\eta}))||_F^2
\end{aligned}
\end{equation}

The discrete solution of $V$ can be directly represented as
\begin{equation}
\small
\begin{aligned}
\label{eq:tz}
V=\texttt{Sgn}(\Theta-\frac{E_{\mu}}{\mu}-\frac{\alpha}{\mu}\Theta A + \frac{\eta}{\mu} U^\texttt{T}(Y-\Gamma+\frac{E_{\eta}}{\eta}))
\end{aligned}
\end{equation}
\noindent where $\texttt{Sgn}(\cdot)$ is signum function which returns -1 if $x<0$, and 1 if $x\geq 0$.

\textbf{Step 7: Update $E_{\eta}$, $E_{\mu}$, $\eta$, $\mu$}. The update rules are
\begin{equation}
\small
\begin{aligned}
\label{eq:tu}
&E_{\eta}=E_{\eta}+\eta(Y-UV-\Gamma)\\
&E_{\mu}=E_{\mu}+\mu(V-\Theta)\\
&\eta=\rho\eta, \mu=\rho\mu\\
\end{aligned}
\end{equation}
$\rho>1$ is learning rate that controls the convergence.



\textbf{Hashing Function Learning.} After obtaining $V$, we substitute the value into $W=(YY^{\texttt{T}}+\gamma I)^{-1}YV^{\texttt{T}}$ and get the linear projection matrix. The hashing functions can be represented as $H(x)=\texttt{sgn}(W^{\texttt{T}}x)$.

\textbf{Online Mobile Landmark Search.} Given a landmark query image $q$ from mobile device, we first extract intermediate representation $Y_q$ as Eq.(\ref{calculatey}). Its binary hashing codes are calculated as $V_q=\texttt{sgn}(W^{\texttt{T}}Y_q)$.
\begin{algorithm}
\caption{CV-DMH based mobile landmark search}
\label{alg:summary}
\begin{algorithmic}[1]
\REQUIRE ~~\\
Query image $q$. Database images $\mathcal I= \{\mathcal I_n\}_{n=1}^N$. \\
\ENSURE ~~\\
Discrete hashing codes of database images $V$. Hashing functions $H$. Image retrieval results for image query $q$.\\
\emph{\textbf{Offline Learning}}
\STATE Extract multi-modal canonical views as illustrated in \ref{sec:3:2}; \\
\STATE Generate intermediate landmark representation as \ref{sec:3:3}; \\
\STATE Compute hashing codes of database images $V$ by solving problem (\ref{disretobject}) with Algorithm \ref{calv}; \\
\STATE Construct hashing functions $H$ with the projection matrix; \\
\emph{\textbf{Online Searching}}
\STATE Extract multi-modal visual features of query image and transform them to a unified intermediate representation; \\
\STATE Project query into hashing codes with the learned $H$; \\
\STATE Perform searching in Hamming space and return results. \\
\end{algorithmic}
\end{algorithm}

\subsection{Summarization and Computational Complexity Analysis}
The key steps of CV-DMH based MLS are described in Algorithm \ref{alg:summary}.
It can be easily derived that the computation cost of multi-modal canonical view mining is $O(NPT)$, as there are $P$ visual modalities
and $T$ iterations for mining. The part of intermediate generation solves a constrained least square fitting problem which also consumes $O(NPT)$.
The computational complexity of discrete optimization is $O(\#iter(TPN+ TPc + cN))$, where $\#iter$ denotes the number of iterations in Algorithm \ref{calv}. Given $N\gg TP>c$, this process scales linearly with $N$. The computation of hashing functions solves a linear system, which consumes $O(N)$.

\section{Experimental Configuration}
\label{sec:4}
\subsection{Experimental Datasets and Setting}
In this paper, three real landmark datasets, \emph{Oxford5K} \cite{CVPRoxford}, \emph{Paris6K} \cite{parisdataset}, and \emph{Paris500K} \cite{iccvlandmark}\footnote{In this experiment, the maximum number of images in each category is limited to 2000 to avoid bias.}, are applied in empirical study. \emph{Oxford5K} is comprised of 5,062 images recorded for 17 landmarks in Oxford.
\emph{Paris6K} consists of 6,412 \emph{Paris} landmark images in 12 categories.
\emph{Paris500K} contains 41,673 images with clustering ground truth which describes
79 landmarks. For \emph{Oxford5K} and \emph{Paris6K}, 10\%, 20\%, and 70\% images
are used as query images, training images, and database images, respectively.
For \emph{Paris500K}, the corresponding ratios are 10\%, 10\%, and 80\%. For three datasets, both
query and database images appear with great visual diversity. Each image is represented by features in 5 heterogenous visual modalities: 81-D Color Moments (CM) \cite{CM}, 58-D Local Binary Pattern (LBP) \cite{LBP}, 80-D Edge Direction Histogram (EDH)
\cite{EDH}, 1,000-D BoVW\footnote{128-D SIFT is employed as local
descriptor.} \cite{bovw}, and 512-D GIST \cite{GIST}.

\subsection{Evaluation Metrics}
\label{sec:4:2}
In our experimental study, mean average precision (mAP) is adopted as the evaluation metric for effectiveness.
The metric has been widely used in literature \cite{surveyhashing}. For a given query, average precision (AP) is calculated as
\begin{equation}
\label{mAP}
\begin{aligned}
AP=\frac{1}{NR}\sum_{r=1}^{R} pre(r)rel(r)
\end{aligned}
\end{equation}
\noindent where $R$ is the total number of retrieved images, $NR$ is the number of relevant images in retrieved set, $pre(r)$ denotes the precision of top $r$ retrieval images, which is defined as the ration between the number of the relevant images and the number of retrieved images $r$, and $rel(r)$ is indicator function which equals to 1 if the $r_{th}$ image is relevant to query, and 0 vice versa. mAP is defined as the average of the AP of all queries. Larger mAP indicates the better retrieval performance. In experiments, we set $R$ as 100 to collect experimental results. Furthermore, \emph{Precision-Scope} curve is also reported to reflect the retrieval performance variations with respect to the number of retrieved images.

\begin{table*}
\small
\caption{mAP of all approaches on three datasets. The best performance in each column is marked with bold.}
\label{resulttable}
\centering
\begin{tabular}{|p{14mm}<{\centering}|p{8mm}<{\centering}|p{8mm}<{\centering}|p{8mm}<{\centering}|p{10mm}<{\centering}
|p{8mm}<{\centering}|p{8mm}<{\centering}|p{8mm}<{\centering}|p{10mm}<{\centering}|p{8mm}<{\centering}|p{8mm}<{\centering}
|p{8mm}<{\centering}|p{10mm}<{\centering}|}
\hline
\multirow{2}{*}{Methods} & \multicolumn{4}{c|}{\emph{Oxford5K}} & \multicolumn{4}{c|}{\emph{Paris6K}} & \multicolumn{4}{c|}{\emph{Paris500K}}\\
\cline{2-13}
& 32 & 48 & 64  & 128  & 32 & 48 & 64  & 128 & 32 & 48 & 64  & 128 \\
\hline
AGH & \textbf{0.3164} & 0.2990 & 0.3098 & 0.3067 & 0.3301 & 0.3283 & 0.3524 & 0.3267 & 0.3427 & 0.3798 & 0.3872 & 0.3944 \\
\hline
ITQ & 0.2778 & 0.2800 & 0.2910 & 0.2917 & 0.2835 & 0.2942 & 0.3059 & 0.3275 & 0.2251 & 0.2765 & 0.3063 & 0.3621 \\
\hline
SGH & 0.2944 & 0.3000 & 0.3190 & 0.3315 & 0.3214 & 0.3497 & 0.3594 & 0.3793 & 0.3242 & 0.3718 & 0.4103 & 0.4620 \\
\hline
DPLM & 0.2985 & 0.3134 & 0.3112 & 0.3065 & 0.2706 & 0.2800 & 0.2962 & 0.3112 & 0.3044 & 0.3490 & 0.3663 & 0.3903 \\
\hline
CHMIS & 0.2959 & 0.3087 & 0.2991 & 0.3211 & 0.3256 & 0.3253 & 0.3265 & 0.3674 & 0.3937 & 0.4429 & 0.4681 & 0.5312 \\
\hline
MVAGH & 0.2947 & 0.3124 & 0.3105 & 0.3011 & 0.2646 & 0.2725 & 0.2882 & 0.3205 & 0.2907 & 0.3181 & 0.3403 & 0.3647 \\
\hline
MFH & 0.2718 & 0.2874 & 0.3027 & 0.3206 & 0.2897 & 0.3066 & 0.3108 & 0.3546 & 0.3537 & 0.4202 & 0.4461 & 0.5211 \\
\hline
CMKH & 0.2999 & 0.3034 & 0.2994 & 0.3026 & \textbf{0.3396} & 0.3351 & 0.3441 & 0.3481 & 0.4094 & 0.4680 & 0.5096 & 0.5493 \\
\hline
MVLH & 0.2918 & 0.3092 & 0.3051 & 0.3203 & 0.3152 & 0.3183 & 0.3531 & 0.3936 & 0.2965 & 0.3433 & 0.3570 & 0.4124 \\
\hline
MAVH & 0.2980 & 0.2856 & 0.2861 & 0.3021 & 0.2975 & 0.3188 & 0.3385 & 0.3565 & 0.2889 & 0.3177 & 0.3353 & 0.3876 \\
\hline
CV-DMH & 0.3028 & \textbf{0.3153} & \textbf{0.3275} & \textbf{0.3458} & 0.3350 & \textbf{0.3534} & \textbf{0.3799} & \textbf{0.4190} & \textbf{0.4533} & \textbf{0.5029} & \textbf{0.5293}  & \textbf{0.6025}\\
\hline
\end{tabular}
\end{table*}

\begin{figure*}
\centering
\mbox{
\subfigure{\includegraphics[width=42mm]{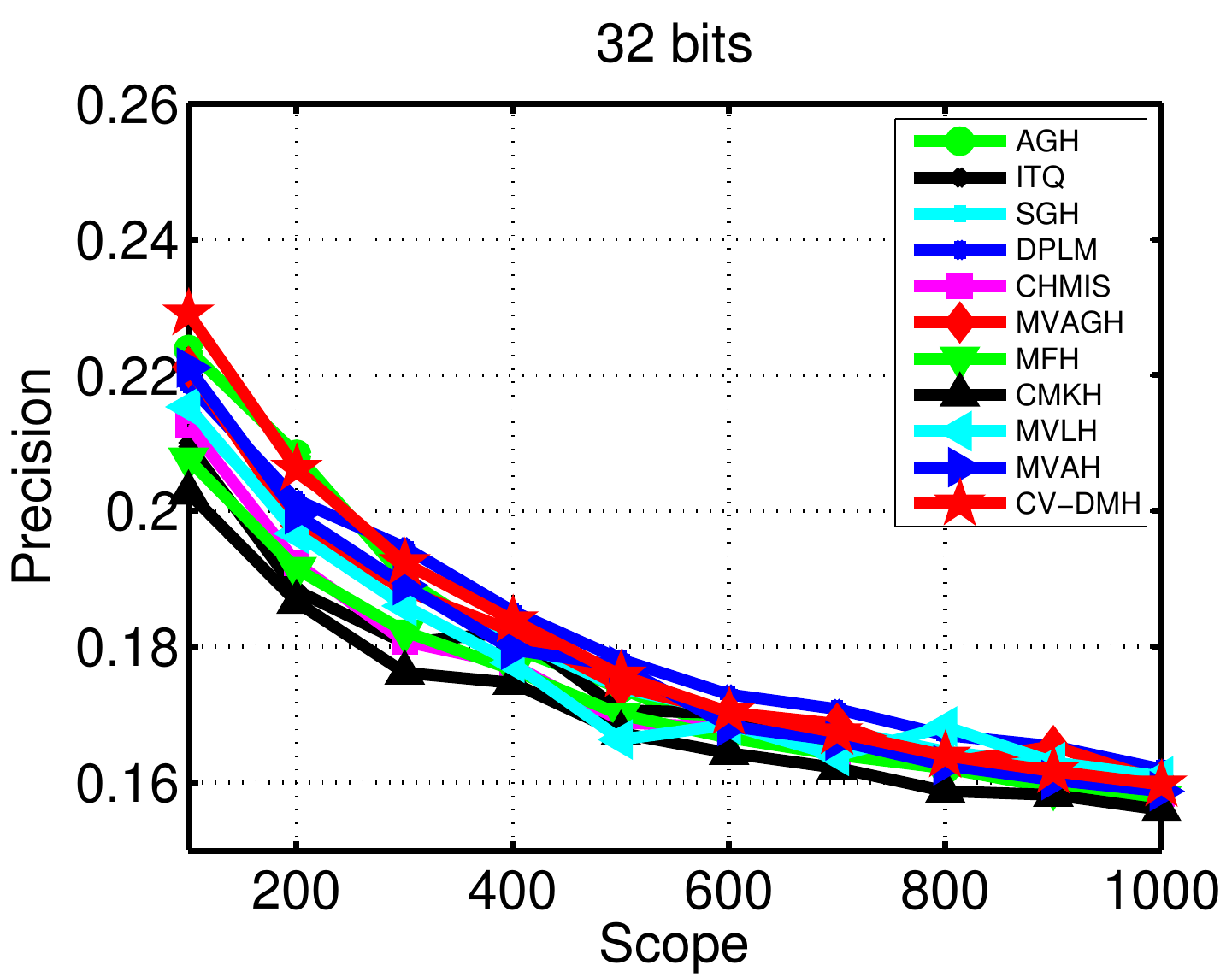}}
}\mbox{
\subfigure{\includegraphics[width=42mm]{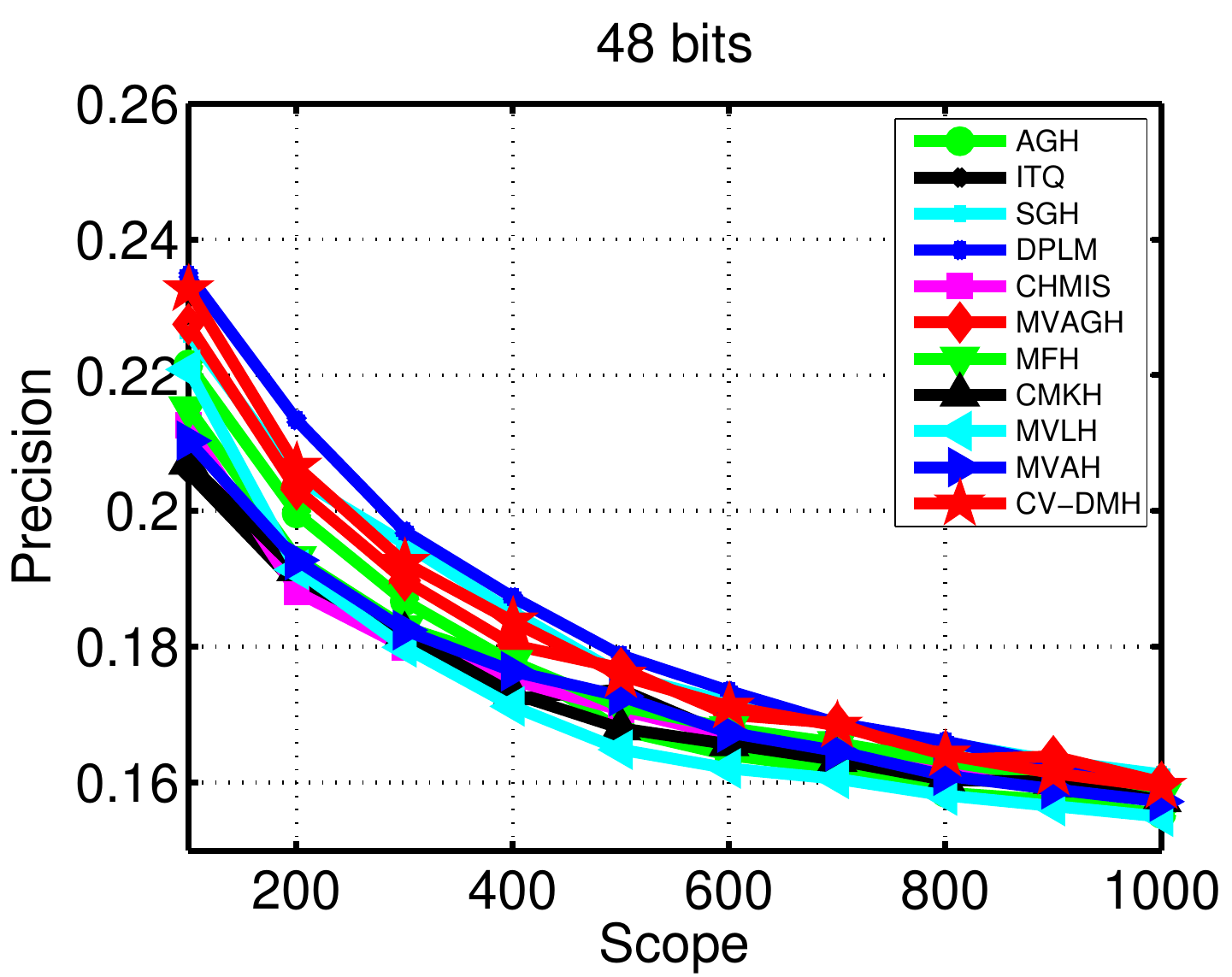}}
}\mbox{
\subfigure{\includegraphics[width=42mm]{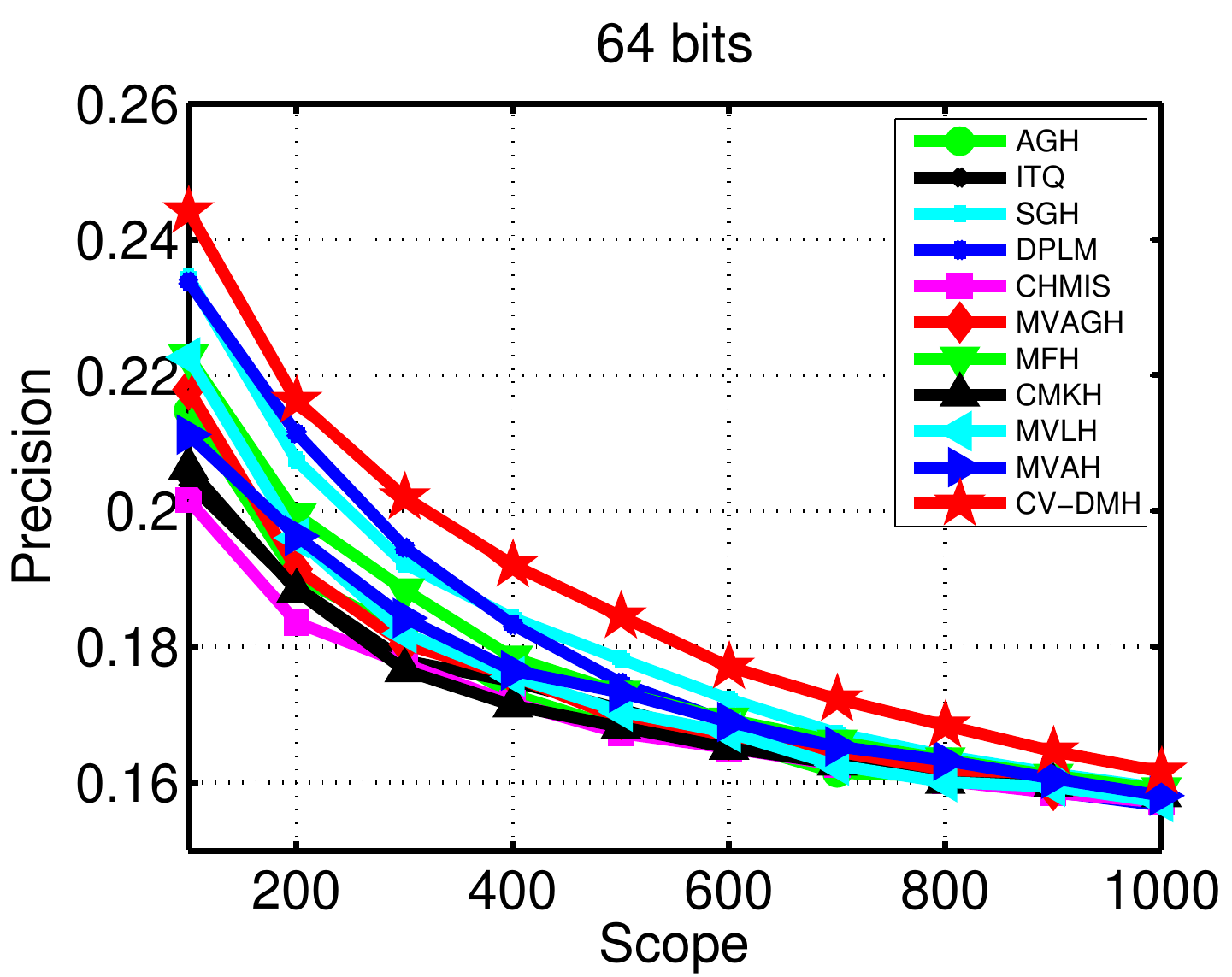}}
}\mbox{
\subfigure{\includegraphics[width=42mm]{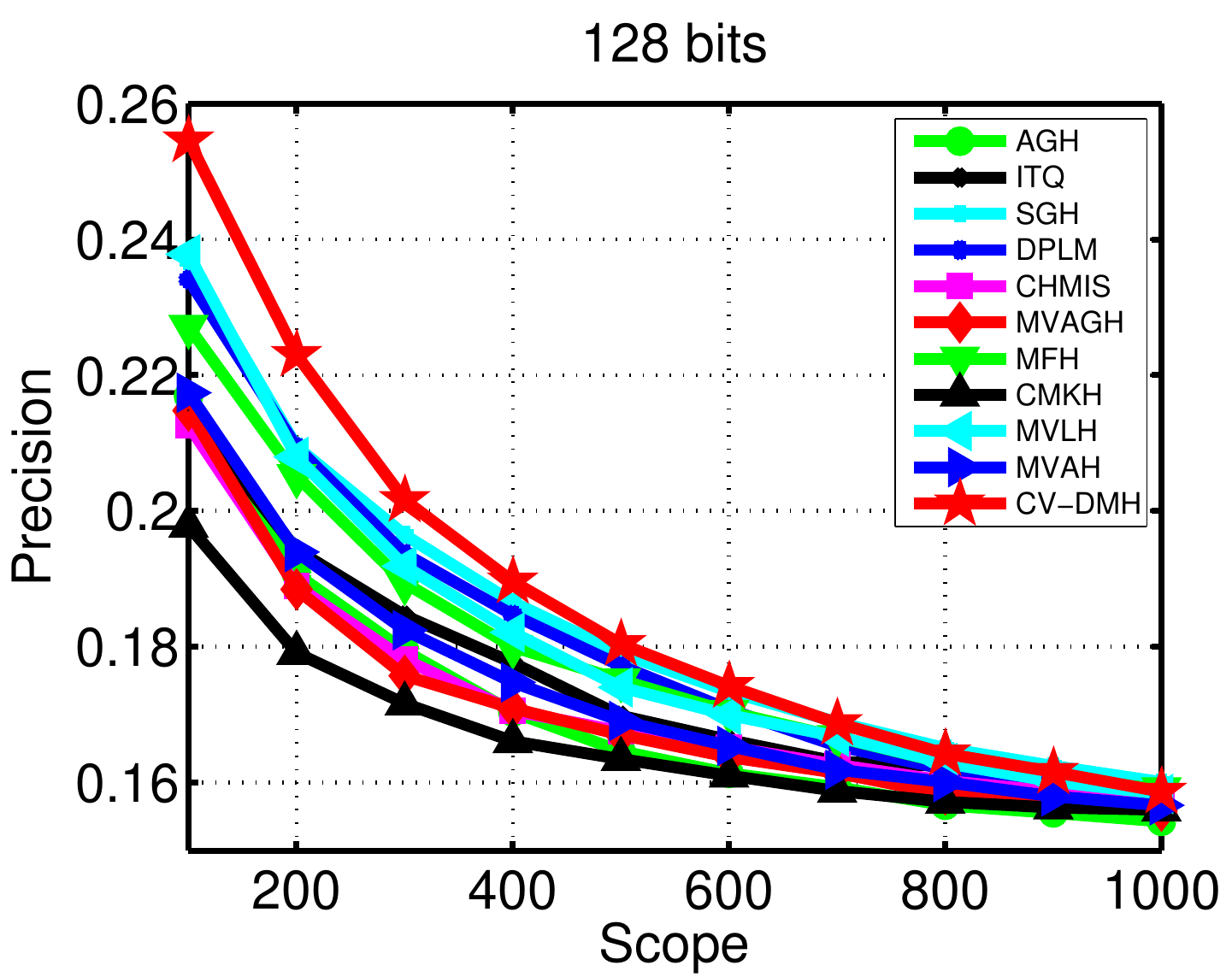}}
}
\vspace{-3mm}
\caption{\emph{Precision-Scope} curves on \emph{Oxford5K} varying code length.}
\label{fig_oxford5k}
\vspace{-3mm}
\end{figure*}
\begin{figure*}
\centering
\mbox{
\subfigure{\includegraphics[width=42mm]{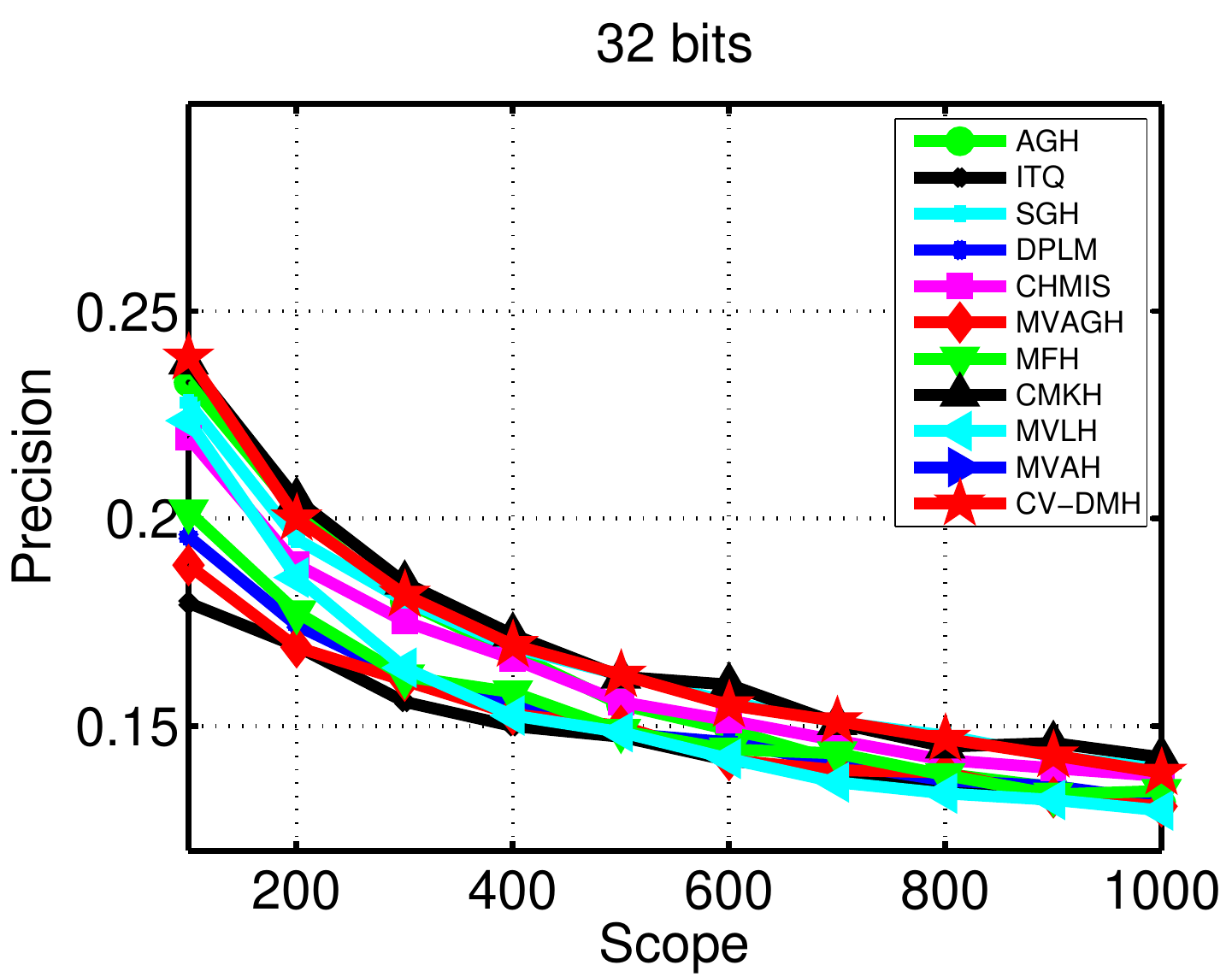}}
}\mbox{
\subfigure{\includegraphics[width=42mm]{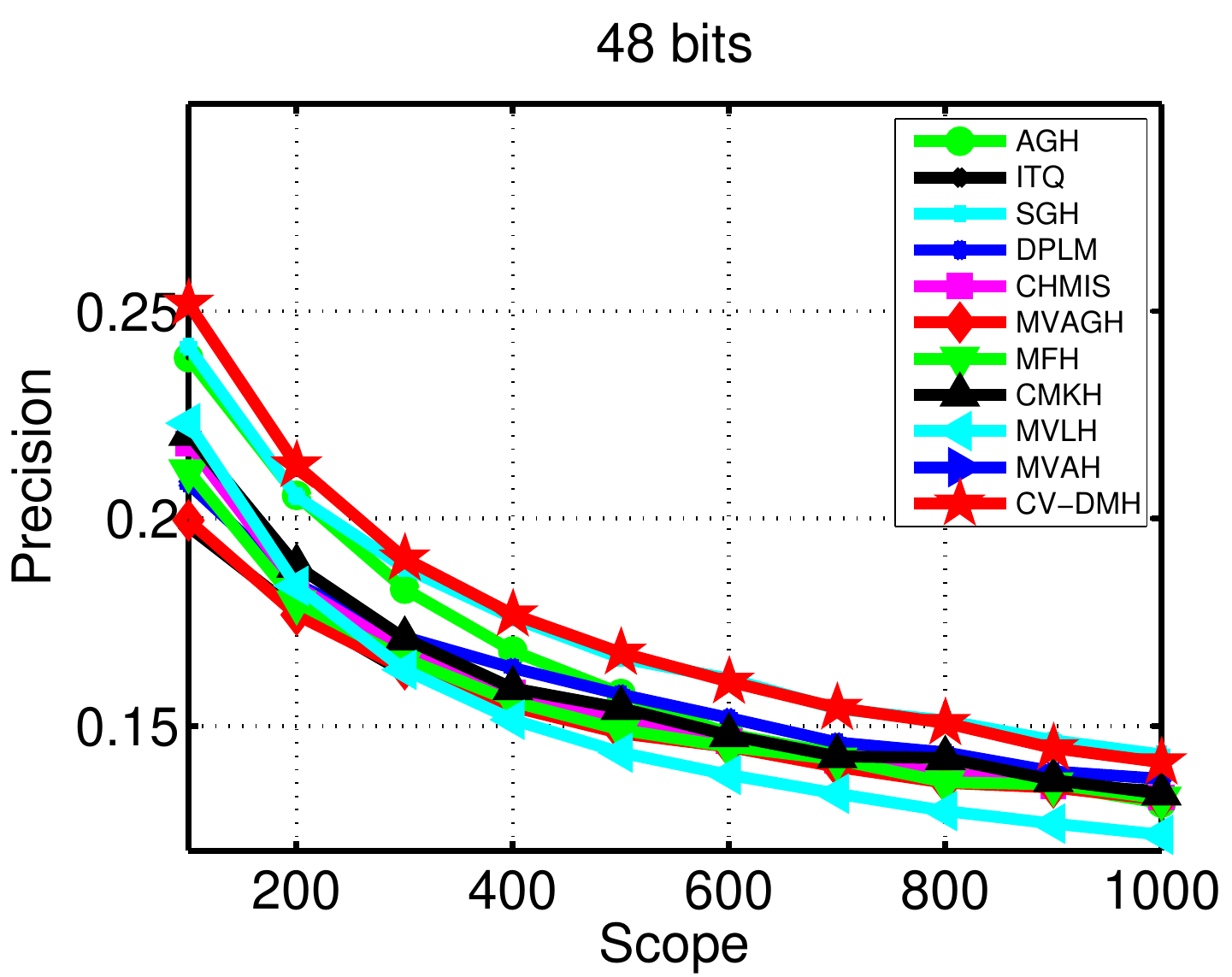}}
}\mbox{
\subfigure{\includegraphics[width=42mm]{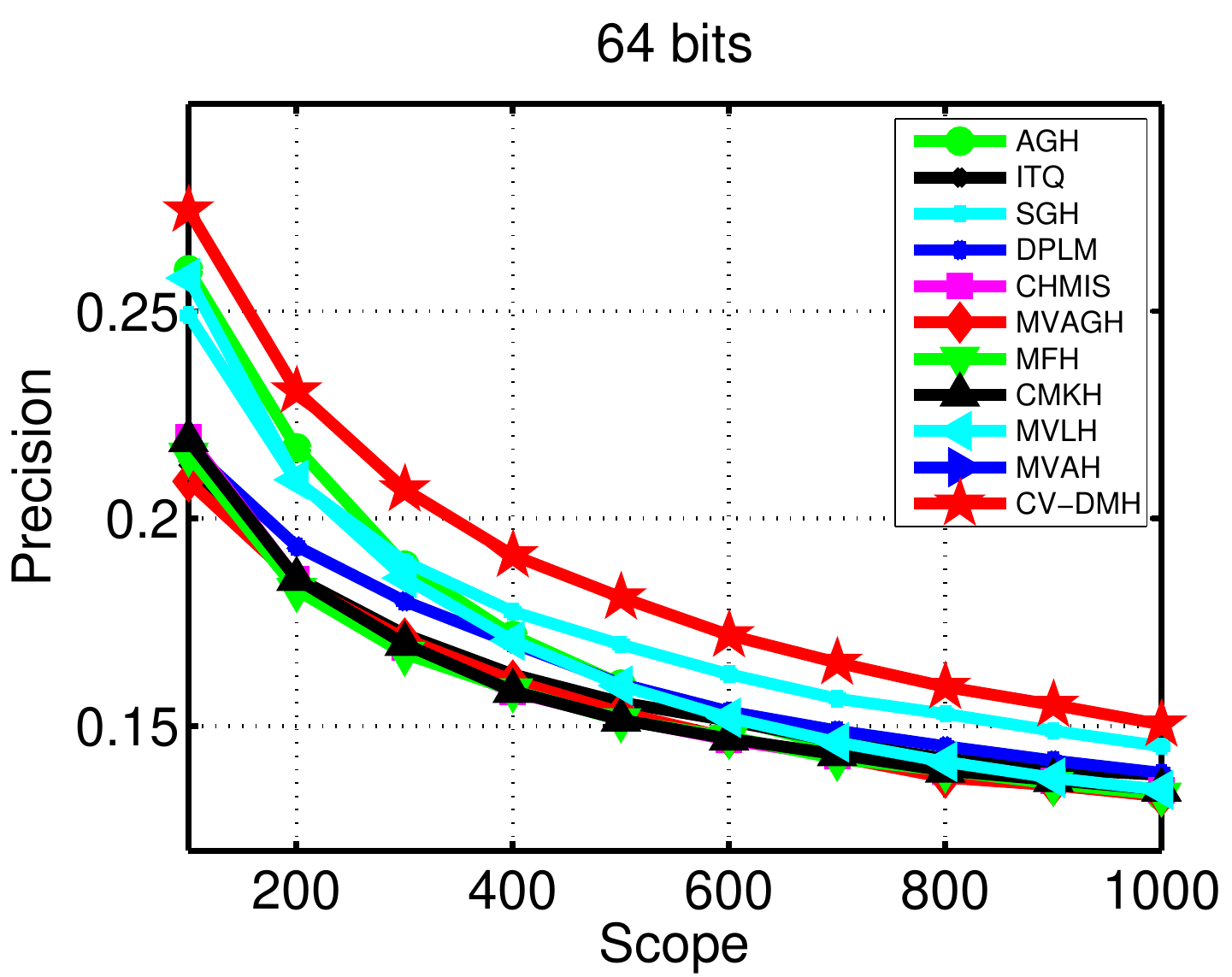}}
}\mbox{
\subfigure{\includegraphics[width=42mm]{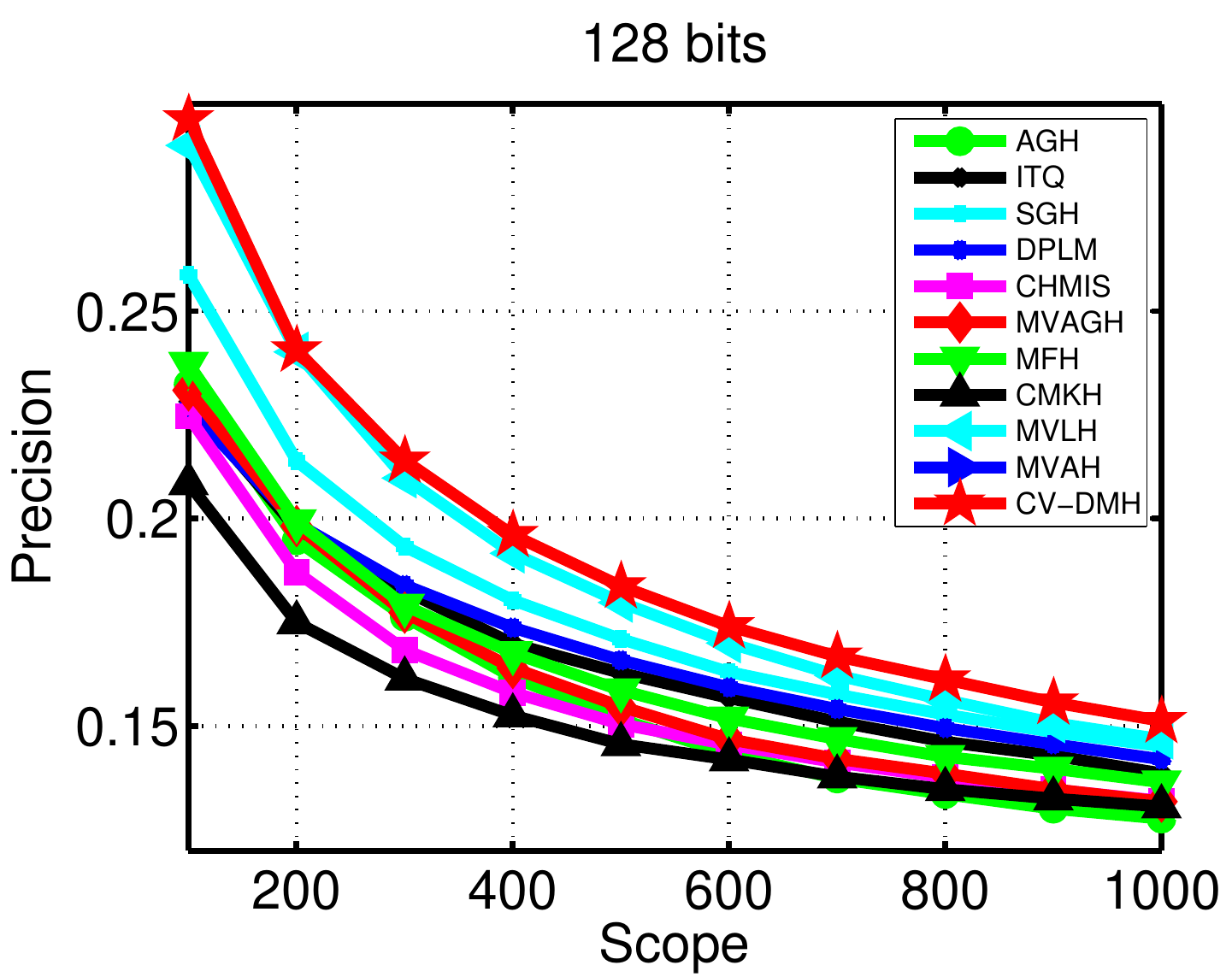}}
}
\vspace{-3mm}
\caption{\emph{Precision-Scope} curves on \emph{Paris6K} varying code length.}
\label{fig_paris6k}
\vspace{-3mm}
\end{figure*}
\begin{figure*}
\centering
\mbox{
\subfigure{\includegraphics[width=42mm]{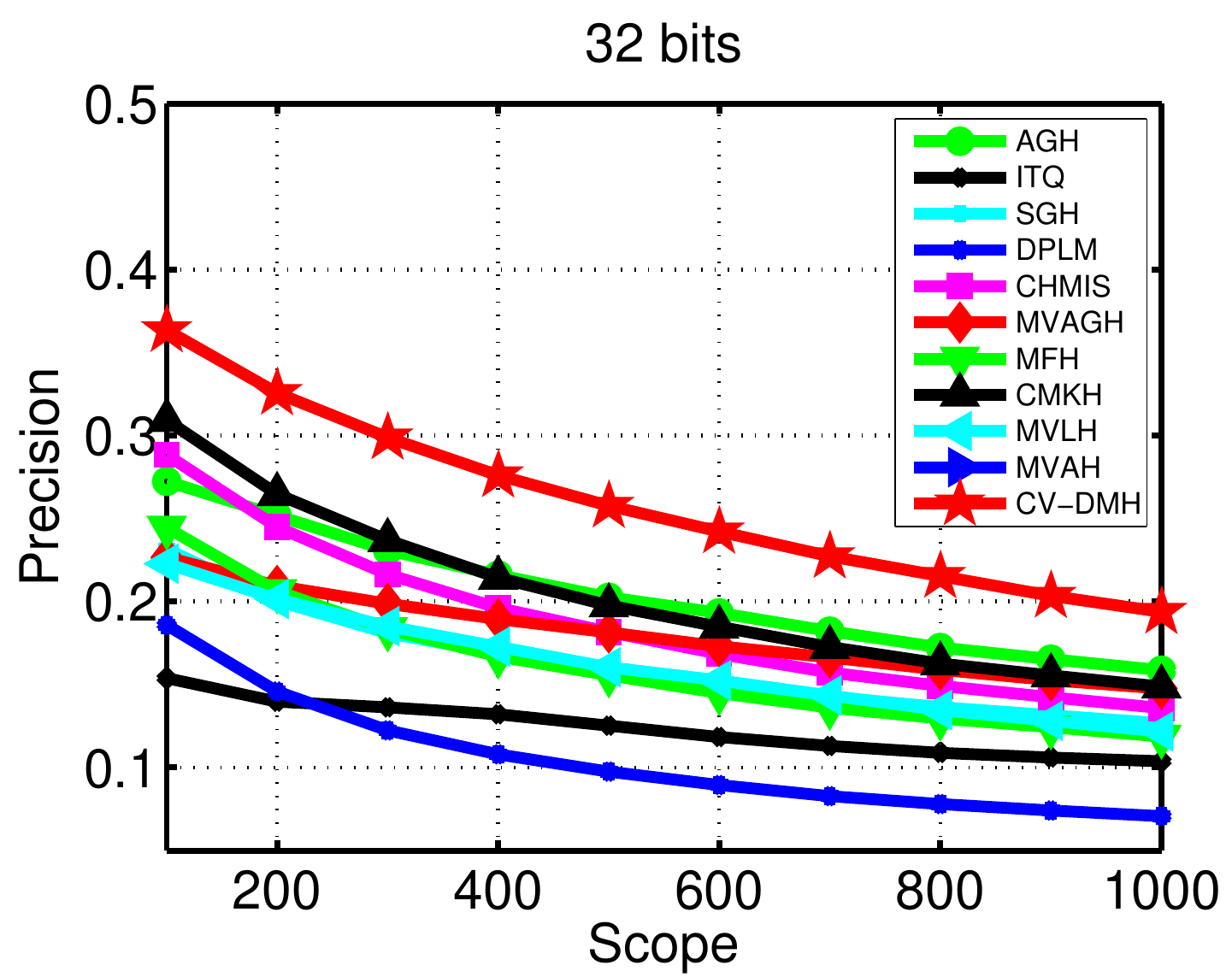}}
}\mbox{
\subfigure{\includegraphics[width=42mm]{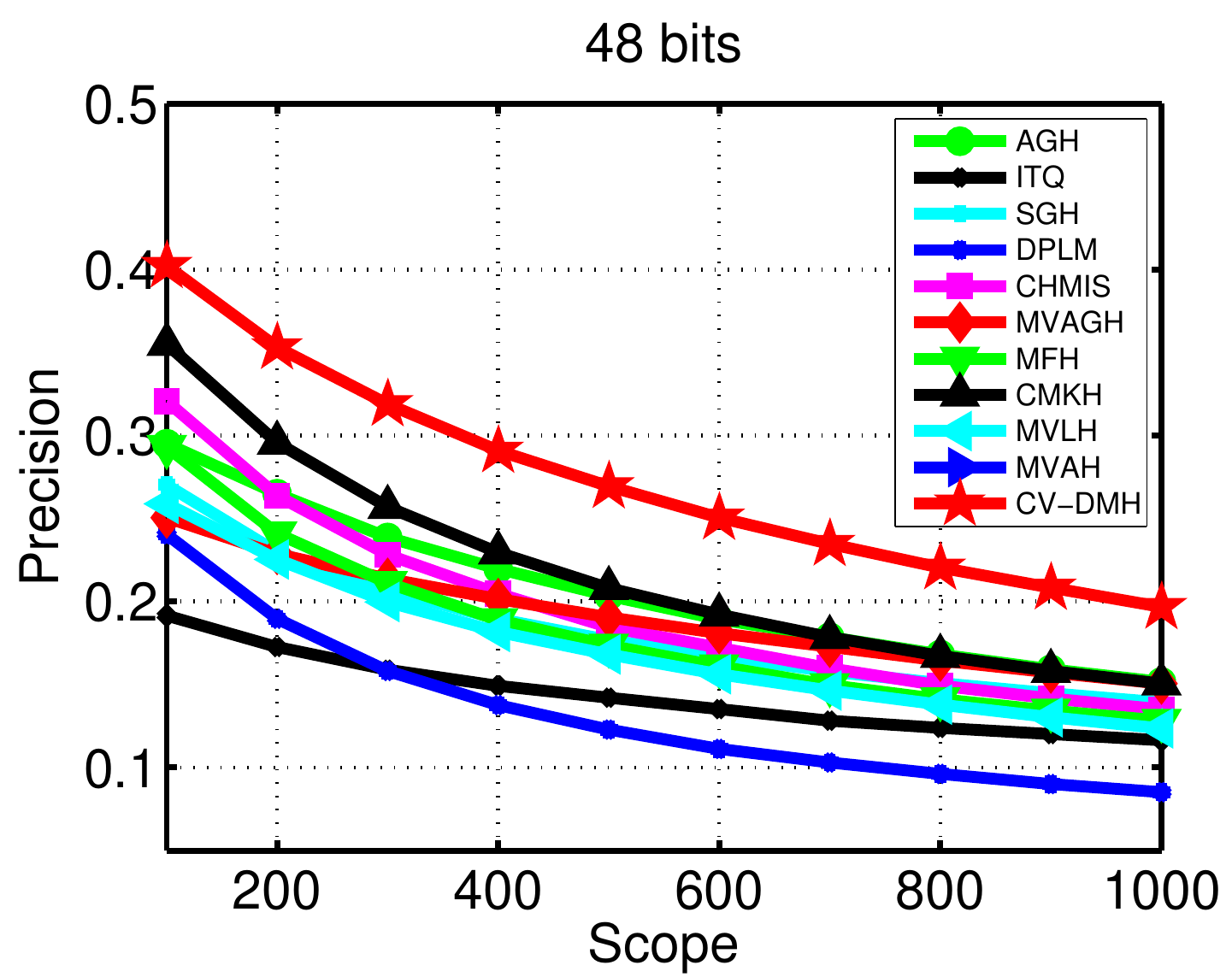}}
}\mbox{
\subfigure{\includegraphics[width=42mm]{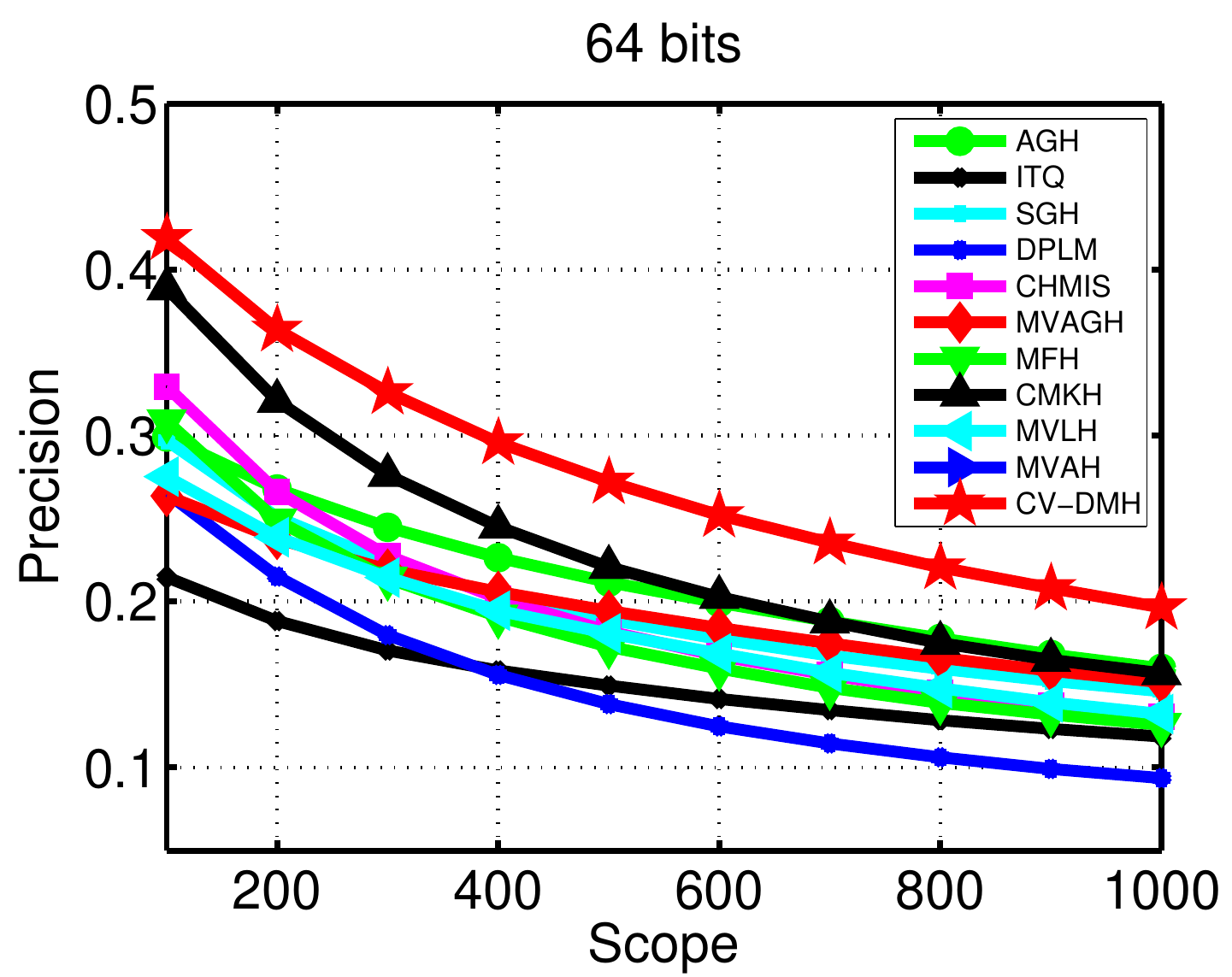}}
}\mbox{
\subfigure{\includegraphics[width=42mm]{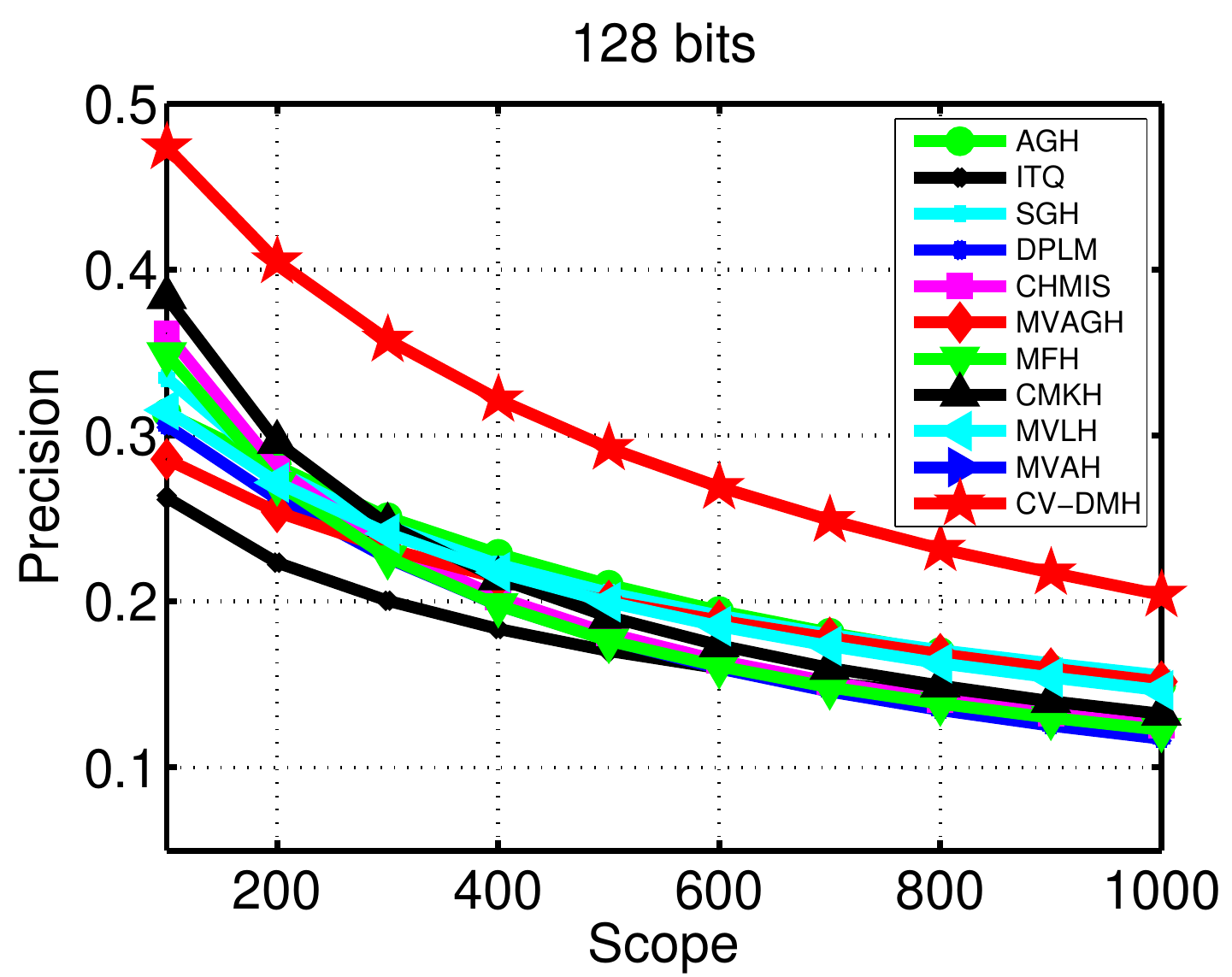}}
}
\vspace{-3mm}
\caption{\emph{Precision-Scope} curves on \emph{Paris500K} varying code length.}
\label{fig_mmaris}
\vspace{-3mm}
\end{figure*}
\subsection{Compared Approaches}
We compare CV-DMH with several state-of-the-art multi-modal hashing approaches.
They include\footnote{For CHMIS, MFH, and CMKH, implementation codes of
them are provided by the authors. For MVAGH, MVLH, and MVAH, we implement them according to the relevant
literature.}:
\begin{itemize}
\itemsep=3pt
  \item Composite hashing with multiple information sources (CHMIS) \cite{HASHMFFOUR}. It is the first work to extend uni-modal spectral hashing to multiple modalities. CHMIS preserves visual similarity with graph and simultaneously integrates information from multiple modalities into the hashing codes with adjusted weights.
  \item Multi-view anchor graph hashing (MVAGH) \cite{MVAGH}. It generates the nonlinearly integrated binary codes as a subset of eigenvectors calculated from an averaged similarity matrix.
  \item Multiple feature hashing (MFH) \cite{HASHMFTWO}. MFH preserves the local structural information of each individual feature modality and also globally considers the local structures for all the features to learn hashing codes.
  \item Compact multiple kernel hashing (CMKH) \cite{HASHMFONE}. CMKH formulates multi-modal hashing as a similarity preserving problem with linearly combined multiple kernels.
  \item Multi-view latent hashing (MVLH) \cite{MVLH}. MVLH first transforms multi-modal features into a unified kernel feature space, then applies matrix factorization to learn the latent factors as the target binary codes.
  \item Multi-view alignment hashing (MVAH) \cite{TIPMVAH}. MVAH first learns multi-modal fused projected vectors with a regularized kernel nonnegative matrix factorization. Then, hashing functions are learned via multivariable logistic regression.
\end{itemize}

Besides, we also compare CV-DMH with several state-of-the-art uni-modal hashing
approaches\footnote{Implementation codes of all these methods are downloaded directly from author websites}: AGH \cite{HASHGRAPH}, ITQ \cite{ITQ}, SGH \cite{SGH} and DPLM \cite{TIP2016binary}. For them, multi-modal features are concatenated into a unified vector
for subsequent learning. The involved parameters of the compared approaches are strictly
adjusted to report the maximum performance with the guidance of relevant literature.

\subsection{Implementation Details}
CV-DMH has three parameters: $\alpha, \beta$, and $\gamma$ in hashing objective function Eq.(\ref{disobj}). They are used to play the balance between the formulated regularization terms. In experiment, we choose the best parameters from $\{10^{-4}, 10^{-2}, 1, 10^2, 10^4\}$. The best performance of CV-DMH is achieved when $\{\alpha=10^{-2}, \beta=10^{-4}, \gamma=10^2\}$, $\{\alpha=10^{-2}, \beta=10^{-2}, \gamma=10^4\}$, $\{\alpha=10^{-2}, \beta=10^{-2}, \gamma=10^2\}$ on \emph{Oxford5K}, \emph{Paris6K}, and \emph{Paris500K} respectively. The parameters $\mu$ and $\eta$ in Eq.(\ref{eq:tj}) are set for ALM based discrete optimization. The best performance is obtained when $\{\mu=1, \eta=1\}$, $\{\mu=0.01, \eta=1\}$, $\{\mu=0.01, \eta=1\}$ on \emph{Oxford5K}, \emph{Paris6K}, and \emph{Paris500K} respectively. In addition, the best canonical view size $T$ is set to 100 on \emph{Oxford5K}
and \emph{Paris6K}, and 300 on \emph{Paris500K}.
The best number of nearest canonical views $r$ in Eq.(\ref{eq:afm}) is set to 70 on \emph{Oxford5K}
and \emph{Paris6K}, and 200 on \emph{Paris500K}. $\sigma$ in Eq.(\ref{eq:afm}) is set to $10^{-4}$ to maximize the performance.

In experiments, hashing code length on all datasets is
varied in the range of [32, 48, 64, 128] to observe the performance. Further,
The retrieval scope on three datasets is set from 100 to 1000 with step size 100. In the first step of Algorithm \ref{calv}, the initial values of $E_{\eta}, E_{\mu}$ are set to 0. The values of $V$ is calculated by PCAH \cite{PCAH}.
\begin{table*}
\small
\caption{Canonical views improve the robustness of CV-DMH. CV-DMH-I denotes direct multi-modal hashing without canonical views.}
\label{effcan}
\centering
\begin{tabular}{|p{16mm}<{\centering}|p{8mm}<{\centering}|p{8mm}<{\centering}|p{8mm}<{\centering}|p{10mm}<{\centering}
|p{8mm}<{\centering}|p{8mm}<{\centering}|p{8mm}<{\centering}|p{10mm}<{\centering}|p{8mm}<{\centering}|p{8mm}<{\centering}
|p{8mm}<{\centering}|p{10mm}<{\centering}|}
\hline
\multirow{2}{*}{Methods} & \multicolumn{4}{c|}{\emph{Oxford5K}} & \multicolumn{4}{c|}{\emph{Paris6K}} & \multicolumn{4}{c|}{\emph{Paris500K}}\\
\cline{2-13}
& 32 & 48 & 64  & 128  & 32 & 48 & 64  & 128 & 32 & 48 & 64  & 128 \\
\hline
CV-DMH-I & 0.2855 & 0.3005 & 0.3053 & 0.3319 & 0.2905 & 0.3177 & 0.3225 & 0.3488 & 0.3558 & 0.4229 & 0.4509 & 0.5222\\
\hline
CV-DMH & \textbf{0.3028} & \textbf{0.3153} & \textbf{0.3275} & \textbf{0.3458} & \textbf{0.3350} & \textbf{0.3534} & \textbf{0.3799} & \textbf{0.4190} & \textbf{0.4533} & \textbf{0.5029} & \textbf{0.5293}  & \textbf{0.6025}\\
\hline
\end{tabular}
\end{table*}
\begin{table*}
\small
\caption{Effects of canonical view mining in multiple modalities.}
\label{feat}
\centering
\begin{tabular}{|p{14mm}<{\centering}|p{8mm}<{\centering}|p{8mm}<{\centering}|p{8mm}<{\centering}|p{10mm}<{\centering}
|p{8mm}<{\centering}|p{8mm}<{\centering}|p{8mm}<{\centering}|p{10mm}<{\centering}|p{8mm}<{\centering}|p{8mm}<{\centering}
|p{8mm}<{\centering}|p{10mm}<{\centering}|}
\hline
\multirow{2}{*}{Methods} & \multicolumn{4}{c|}{\emph{Oxford5K}} & \multicolumn{4}{c|}{\emph{Paris6K}} & \multicolumn{4}{c|}{\emph{Paris500K}}\\
\cline{2-13}
& 32 & 48 & 64  & 128  & 32 & 48 & 64  & 128 & 32 & 48 & 64  & 128 \\
\hline
CM & 0.2131 & 0.2188 & 0.2192 & 0.2287 & 0.1954 & 0.1971 & 0.2043 & 0.2230 & 0.1511 & 0.1782 & 0.1945 & 0.2451\\
\hline
LBP & 0.2619 & 0.2695 & 0.2872 & 0.2950 & 0.2249 & 0.2284 & 0.2504 & 0.2582 & 0.2585 & 0.3107 & 0.3422 & 0.3899 \\
\hline
EDH & 0.2465 & 0.2658 & 0.2663 & 0.2645 & 0.2224 & 0.2325 & 0.2380 & 0.2620 & 0.2291 & 0.2821 & 0.3147 & 0.3965 \\
\hline
BOVW & 0.2932 & 0.3155 & 0.3022 & 0.3300 & 0.3060 & 0.3252 & 0.3380 & 0.3776 & 0.3255 & 0.3959 & 0.4381 & 0.5112 \\
\hline
GIST & 0.2589 & 0.2634 & 0.2675 & 0.2781 & 0.2517 & 0.2669 & 0.2797 & 0.3042 & 0.2528 & 0.3200 & 0.3745 & 0.4695 \\
\hline
CV-DMH & \textbf{0.3028} & \textbf{0.3153} & \textbf{0.3275} & \textbf{0.3458} & \textbf{0.3350} & \textbf{0.3534} & \textbf{0.3799} & \textbf{0.4190} & \textbf{0.4533} & \textbf{0.5029} & \textbf{0.5293}  & \textbf{0.6025}\\
\hline
\end{tabular}
\end{table*}
\begin{table*}
\small
\caption{Effects of submodular function based canonical view discovery.}
\label{effesumbo}
\centering
\begin{tabular}{|p{18mm}<{\centering}|p{8mm}<{\centering}|p{8mm}<{\centering}|p{8mm}<{\centering}|p{10mm}<{\centering}
|p{8mm}<{\centering}|p{8mm}<{\centering}|p{8mm}<{\centering}|p{10mm}<{\centering}|p{8mm}<{\centering}|p{8mm}<{\centering}
|p{8mm}<{\centering}|p{10mm}<{\centering}|}
\hline
\multirow{2}{*}{Methods} & \multicolumn{4}{c|}{\emph{Oxford5K}} & \multicolumn{4}{c|}{\emph{Paris6K}} & \multicolumn{4}{c|}{\emph{Paris500K}}\\
\cline{2-13}
& 32 & 48 & 64  & 128  & 32 & 48 & 64  & 128 & 32 & 48 & 64  & 128 \\
\hline
\emph{Random} & 0.2740 & 0.2795 & 0.2727 & 0.3097 & 0.3254 & 0.3394 & 0.3529 & 0.3859 & 0.4028 & 0.4415 & 0.4659 & 0.5206 \\
\hline
\emph{K-means} & 0.2648 & 0.2877 & 0.2829 & 0.3199 & 0.3102 & 0.3267 & 0.3466 & 0.3879 & 0.3944 & 0.4418 & 0.4739 & 0.5363 \\
\hline
DL & 0.2523 & 0.2683 & 0.2699 & 0.2848 & 0.2615 & 0.2807 & 0.2822 & 0.3191 & 0.3306 & 0.3568 & 0.3770 & 0.4440 \\
\hline
R-CV-DMH & 0.2890 & 0.3036 & 0.3151 & 0.3434 & 0.3350 & \textbf{0.3627} & \textbf{0.3832 }& 0.4181 & 0.4443 & 0.4916 & 0.5198 & 0.5923\\
\hline
CV-DMH & \textbf{0.3028} & \textbf{0.3153} & \textbf{0.3275} & \textbf{0.3458} & \textbf{0.3350} & 0.3534 & 0.3799 & \textbf{0.4190} & \textbf{0.4533} & \textbf{0.5029} & \textbf{0.5293}  & \textbf{0.6025}\\
\hline
\end{tabular}
\end{table*}
\section{Experimental Results and Discussions}
\label{sec:5}
\subsection{Performance Comparison Results}
We report mAP results and \emph{Precision-Scope}
curves of all approaches in Table \ref{resulttable} and Figure \ref{fig_oxford5k}, \ref{fig_paris6k}, \ref{fig_mmaris},
respectively. From the presented results, we can easily find that CV-DMH outperforms the
competitors on almost all cases. It is interesting to find that, even with less binary bits,
CV-DMH can still achieve higher mAP than many competitors with longer binary codes.
Further, Figure \ref{fig_mmaris} shows that, on \emph{Paris500K} and 128 bits, the precision gain of CV-DMH over the second best
approach is more than 10\%, and it becomes larger when more images are returned. Moreover,
we observe that performance improvement on \emph{Paris500K} is more than that obtained
on \emph{Oxford5K} and \emph{Paris6K}. Indeed, images in \emph{Paris500K} have more landmark images and
diverse visual appearances. This experimental phenomenon validates the desirable property of CV-DMH on
accommodating the visual variations. Finally, we observe that the retrieval performance of CV-DMH on \emph{Oxford5K}
is steadily improved when binary code length increases. However, we don't gain similar
observations for many approaches studied in this experimental study.
This is because CV-DMH adopts discrete optimization to solve hashing codes directly without any relaxing. The design can
successfully avoid the accumulated quantization errors brought in compared approaches. In this case, more binary bits will bring more discriminative information and thus enable CV-DMH to gain higher discriminative capability.

\begin{table*}
\small
\caption{Canonical views improve the robustness of CV-DMH. CV-DMH-II denotes direct multi-modal hashing with anchor transformation.}
\label{effinter}
\centering
\begin{tabular}{|p{16mm}<{\centering}|p{8mm}<{\centering}|p{8mm}<{\centering}|p{8mm}<{\centering}|p{10mm}<{\centering}
|p{8mm}<{\centering}|p{8mm}<{\centering}|p{8mm}<{\centering}|p{10mm}<{\centering}|p{8mm}<{\centering}|p{8mm}<{\centering}
|p{8mm}<{\centering}|p{10mm}<{\centering}|}
\hline
\multirow{2}{*}{Methods} & \multicolumn{4}{c|}{\emph{Oxford5K}} & \multicolumn{4}{c|}{\emph{Paris6K}} & \multicolumn{4}{c|}{\emph{Paris500K}}\\
\cline{2-13}
& 32 & 48 & 64  & 128  & 32 & 48 & 64  & 128 & 32 & 48 & 64  & 128 \\
\hline
CV-DMH-II & 0.2829 & 0.2816 & 0.2903 & 0.2994 & 0.2699 & 0.2768 & 0.2852 & 0.2822 & 0.2877 & 0.3248 & 0.3450 & 0.3750 \\
\hline
CV-DMH & \textbf{0.3028} & \textbf{0.3153} & \textbf{0.3275} & \textbf{0.3458} & \textbf{0.3350} & \textbf{0.3534} & \textbf{0.3799} & \textbf{0.4190} & \textbf{0.4533} & \textbf{0.5029} & \textbf{0.5293}  & \textbf{0.6025}\\
\hline
\end{tabular}
\end{table*}
\begin{table*}
\small
\caption{Effects of discrete optimization. CV-DMH-II: it only considers  bit-uncorrelated constraint by relaxing the discrete constraint and removing balance constraint in the Eq.(\ref{disobj}). CV-DMH-III: it considers bit-uncorrelated constraint and balance constraint together with discrete constraint relaxing. }
\label{effdsop}
\centering
\begin{tabular}{|p{18mm}<{\centering}|p{8mm}<{\centering}|p{8mm}<{\centering}|p{8mm}<{\centering}|p{10mm}<{\centering}
|p{8mm}<{\centering}|p{8mm}<{\centering}|p{8mm}<{\centering}|p{10mm}<{\centering}|p{8mm}<{\centering}|p{8mm}<{\centering}
|p{8mm}<{\centering}|p{10mm}<{\centering}|}
\hline
\multirow{2}{*}{Methods} & \multicolumn{4}{c|}{\emph{Oxford5K}} & \multicolumn{4}{c|}{\emph{Paris6K}} & \multicolumn{4}{c|}{\emph{Paris500K}}\\
\cline{2-13}
& 32 & 48 & 64  & 128  & 32 & 48 & 64  & 128 & 32 & 48 & 64  & 128 \\
\hline
CV-DMH-III & 0.2905 & 0.3179 & 0.3160 & 0.3403 & 0.3110 & 0.3194 & 0.3397 & 0.3887 & 0.4423 & 0.4918 & 0.5189 & 0.5917 \\
\hline
CV-DMH-IV & 0.2932 & 0.3177 & 0.3219 & 0.3373 & 0.3175 & 0.3510 & 0.3510 & 0.4003 & 0.4421 & 0.4935 & 0.5205 & 0.5915 \\
\hline
CV-DMH & \textbf{0.3028} & \textbf{0.3153} & \textbf{0.3275} & \textbf{0.3458} & \textbf{0.3350} & \textbf{0.3534} & \textbf{0.3799} & \textbf{0.4190} & \textbf{0.4533} & \textbf{0.5029} & \textbf{0.5293}  & \textbf{0.6025}\\
\hline
\end{tabular}
\end{table*}
\subsection{Canonical View or Not?}
To see how the canonical view mining can benefit multi-modal hashing learning, we first compare the performance of CV-DMH with the one (denoted as CV-DMH-II) which performs hashing (Eq.(\ref{disobj})) directly on raw concatenated multiple low-level features. Table \ref{effcan} presents the detailed comparative results. From it, we easily find that CV-DMH can consistently yield better performance. On three datasets, the maximum search precision improvements reach about 2\%, 7\%, and 9\%, respectively. The performance improvement is attributed to the fact that, canonical views capture key visual contents of landmarks, diverse visual contents can be robustly accommodated with intermediate representation, and thus hashing codes learned on intermediate representation enjoy better robustness.

Then, we investigate the effects of canonical view mining in multiple modalities. We compare the performance of CV-DMH with the approaches that perform hashing (Eq.(\ref{disobj})) on only uni-modal canonical view set. We denote them directly with the corresponding modality names: CM, LBP, EDH, BOVW, and GIST respectively. Table \ref{feat} presents the main results. It demonstrates that CV-DMH can achieve the best performance. The reason is that, with multi-modal learning, canonical view set can cover more visual variations and thus CV-DMH can enjoy better robustness. All the above results clearly demonstrate that CV-DMH adopts a reasonable strategy by employing canonical views for MLS.

Finally, we validate the effects of the proposed submodular function based canonical view selection approach. We compare the performance of CV-DMH with four variants of our method. The first two competitors \emph{random} and \emph{K-means} discover canonical views by randomly sampling and K-means respectively. The third competitor DL learns canonical views by dictionary learning with K-SVD \cite{Aharon2006}. The last one R-CV-DMH selects canonical views by only considering representativeness of view set. The detailed comparison results are presented in Table \ref{effesumbo}. It can be easily observed that CV-DMH can achieve better performance in most cases. These results demonstrate the effectiveness of submodular function on discovering canonical views of landmarks. The reasons for the weakness of competitors can be explained as follows: the approach \emph{random} discovers canonical views without considering any underlying visual landmark distributions. \emph{K-means} and R-CV-DMH only consider representativeness of the selected view set, they unfortunately ignore the visual redundancy of views. In this case, with the fixed size of canonical view set, real representative views may be excluded by redundant ones. DL seeks the optimal ``dictionary" views that fit the visual distributions of training data. However, the learned dictionaries may not well adapt the query visual contents. In contrast, our approach comprehensively considers representativeness and redundancy of view set. The canonical views capture the key visual characteristics of landmarks and thus can characterize the visual contents of both query and database images.

\subsection{Effects of Intermediate Representation}
Intermediate representation bridges the canonical view mining and discrete binary embedding model. It is generated by calculating the multi-modal sparse reconstruction coefficients between image and canonical views. This subsection evaluates its effectiveness. Specifically, we compare CV-DMH with the variant approach CV-DMH-II that generates intermediate representation by calculating the similarities between images and canonical views. The calculating process is similar to anchor transformation in many existing hashing approaches \cite{SDH,DGH,HASHGRAPH}. In this experiment, we set the same number of nearest canonical views in two compared approaches for calculation. Table \ref{effinter} summarizes the main results. It clearly demonstrates that CV-DMH can consistently outperform the CV-DMH-II on all datasets and code lengths. The potential reason is that: the auto-generated sparse reconstruction coefficient can robustly accommodate the visual variations of landmark images with certain canonical views, they can provide more robust representation bases for subsequent binary embedding.

\subsection{Effects of Discrete Optimization}
To evaluate the effects of discrete optimization, we compare the performance between CV-DMH and its two variants. CV-DMH-III:  it only considers  bit-uncorrelated constraint by relaxing the discrete constraint and removing balance constraint in the Eq.(\ref{disobj}). CV-DMH-IV: it considers bit-uncorrelated constraint and balance constraint together with discrete constraint relaxing. Both two variants adopt conventional ``relaxing+rounding" optimization as many existing hashing approaches. The relaxed hashing values are also solved with ALM and the final binary hashing codes are generated by mean thresholding. Table \ref{effdsop} gives main mAP comparison results. We can clearly observe that CV-DMH can achieve better performance in all cases. These results validate the effects of direct discrete optimization and considering three constraints.

\begin{figure*}
\centering
\mbox{
\subfigure[$\alpha$ is fixed to $10^{-2}$]{\includegraphics[width=40mm]{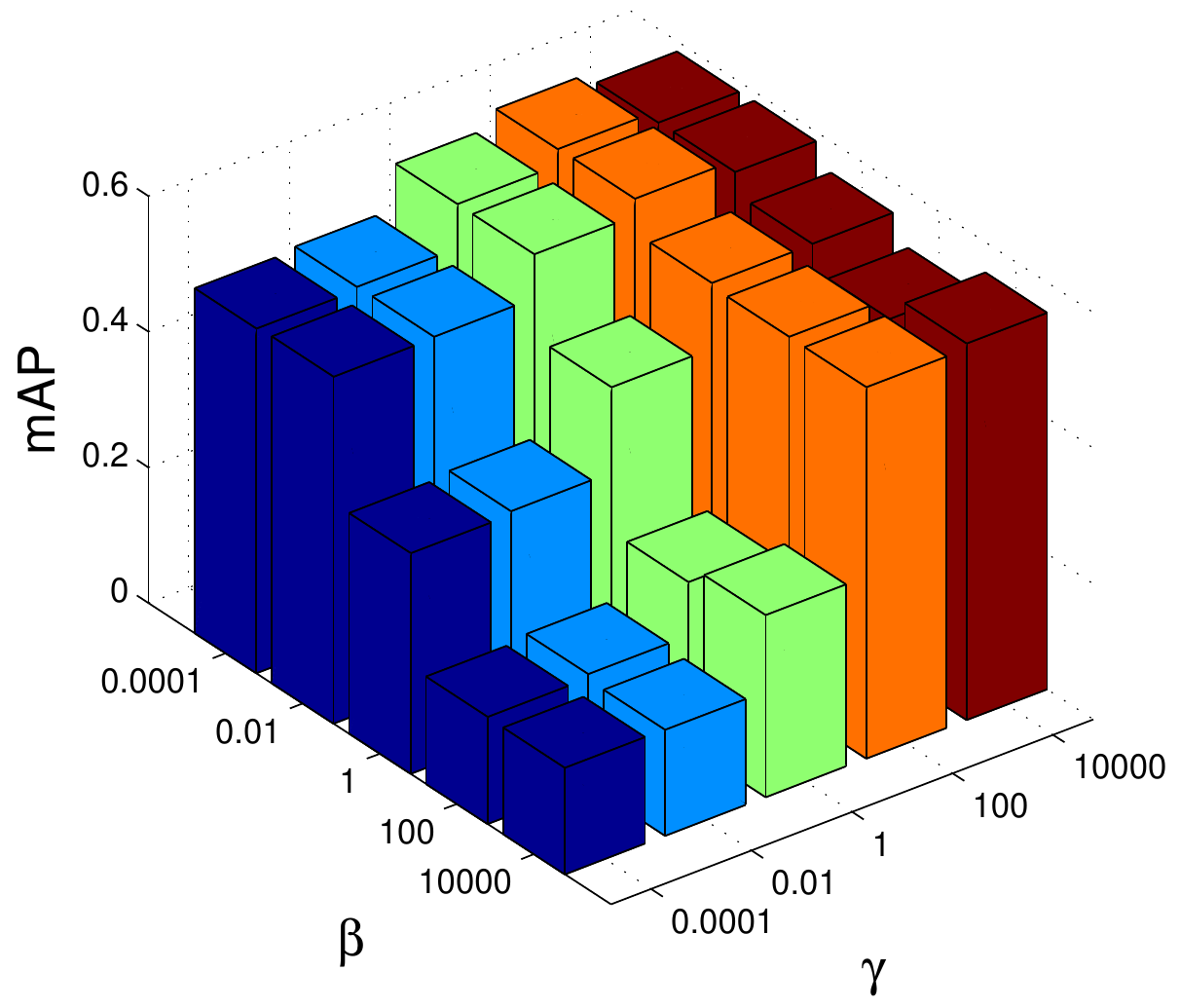}}
}\mbox{
\subfigure[$\beta$ is fixed to $10^{-4}$]{\includegraphics[width=40mm]{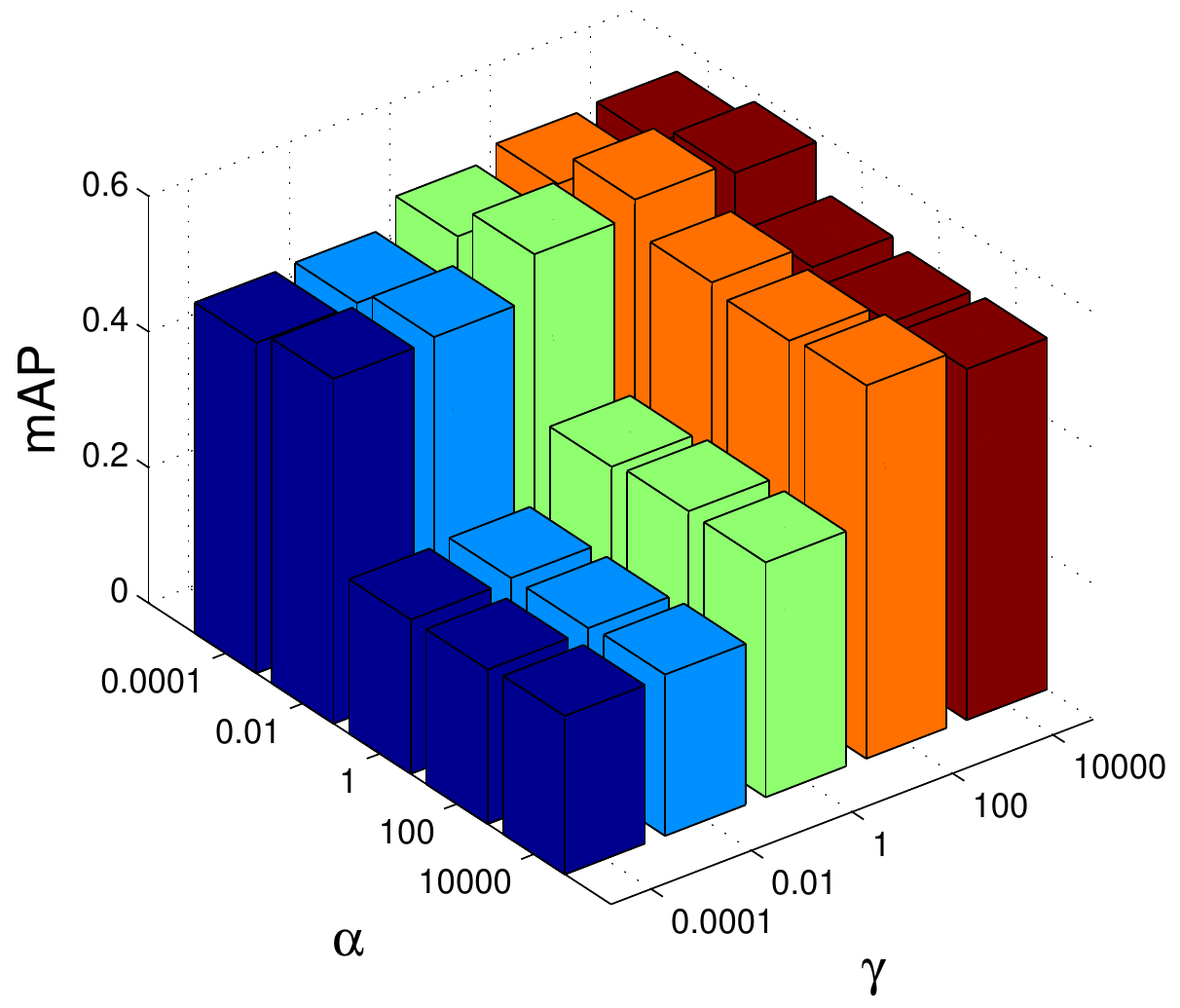}}
}\mbox{
\subfigure[$\gamma$ is fixed to $10^2$]{\includegraphics[width=40mm]{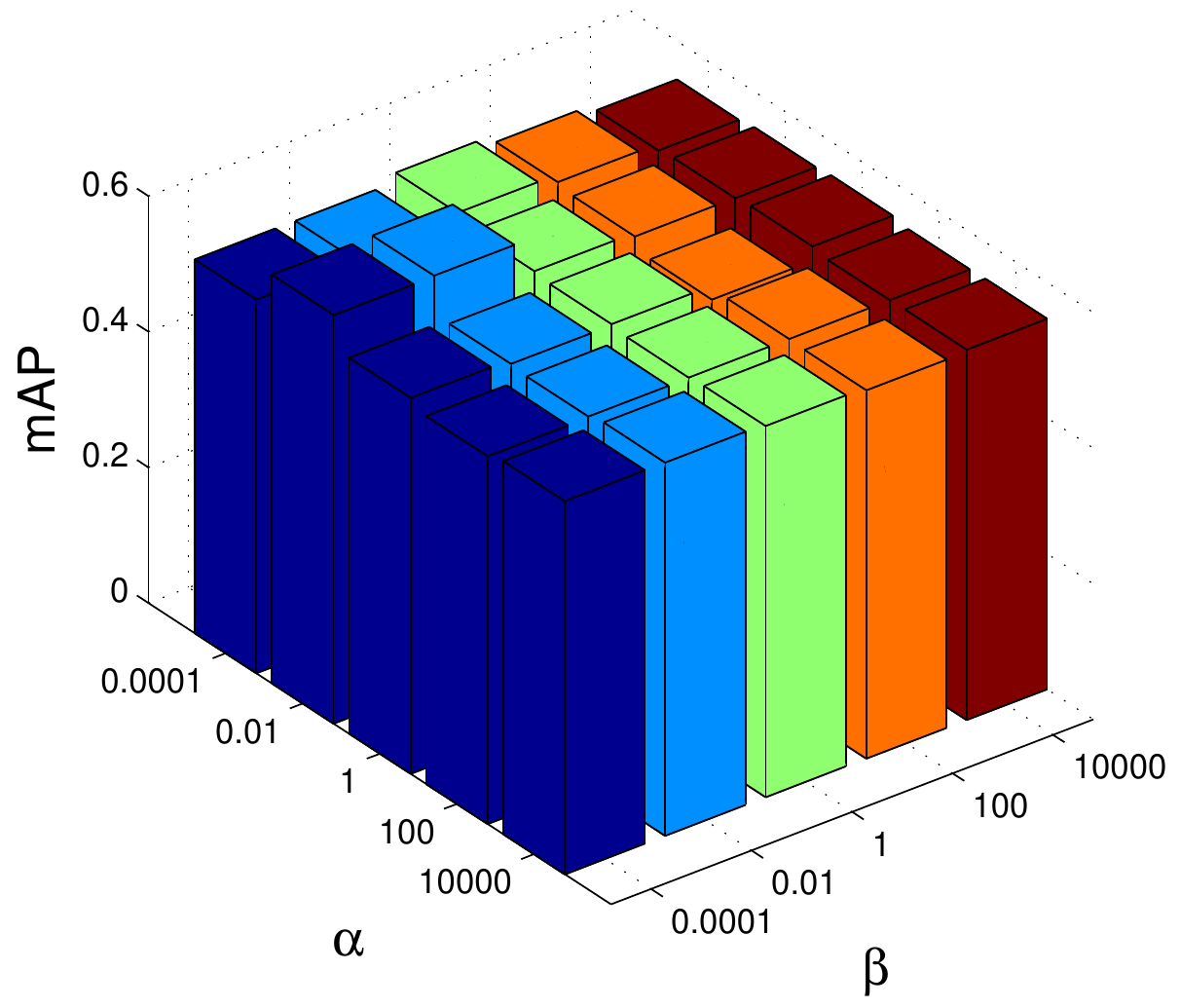}}
}\mbox{
\subfigure[]{\includegraphics[width=40mm]{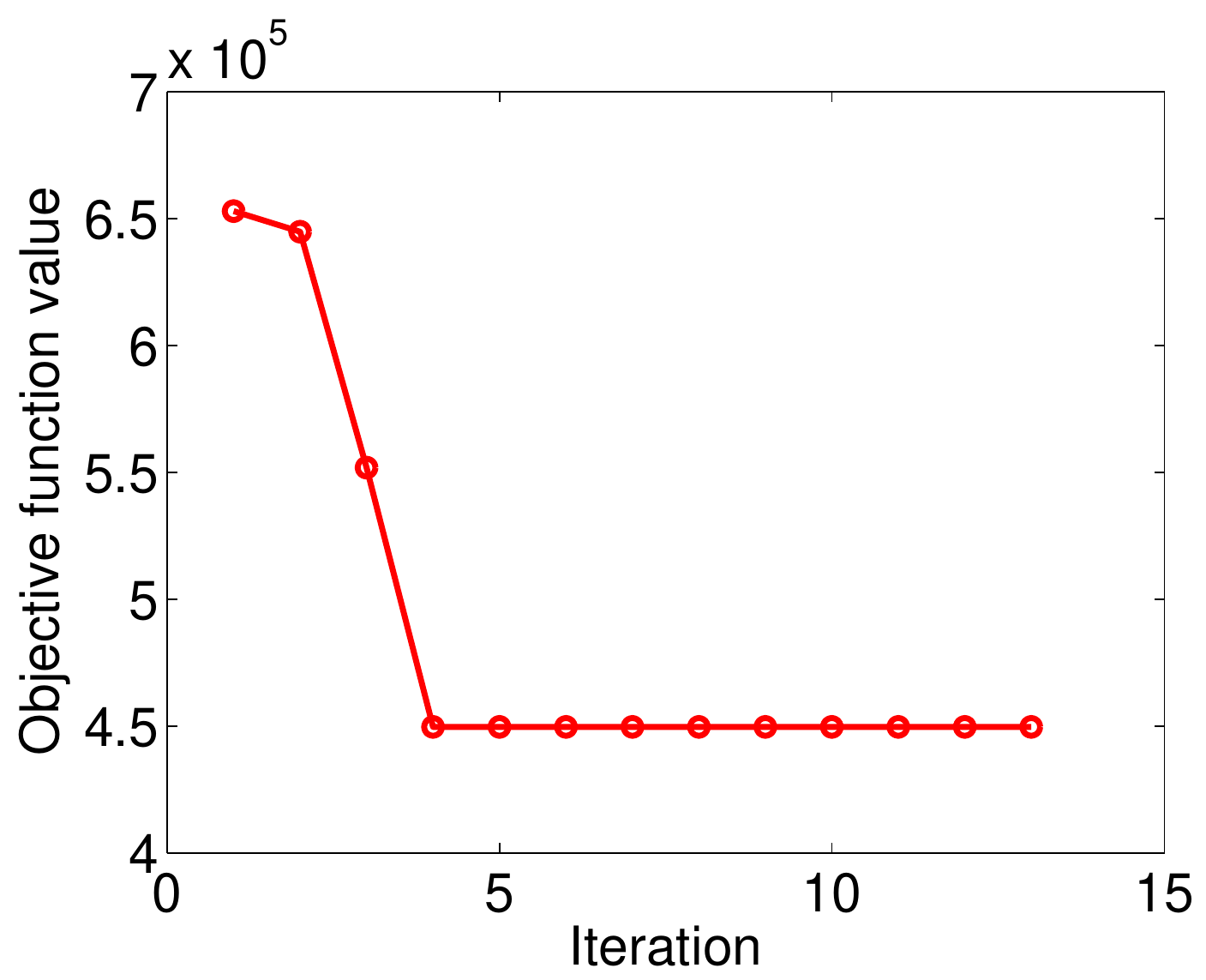}}
}
\caption{(a-c) CV-DMH performance variations with parameters in Eq.(\ref{disobj}) on \emph{Paris5K} when binary code length is 128. (d) Performance performance variations with iterations in Algorithm \ref{calv}.}
\label{fig_robustness}
\end{figure*}

\subsection{Convergency Analysis}
At each iteration in Algorithm \ref{calv}, the updating of variables will decrease the objective function value. As indicated by ALM optimization theory \cite{ALMone}, the iterations will make the optimization process converged. We also conduct empirical study on the convergence property using \emph{Paris500k}. Fig.\ref{fig_robustness}(d) presents the results. We observe that the objective function value first decreases with the number of iterations and then becomes steady after certain iterations (about 7 $\sim$ 8 iterations). This result demonstrates that the convergence of the proposed method.

\subsection{Parameter Study}
In this subsection, we conduct empirical experiments to study the performance variations with involved parameters in CV-DMH. Specifically, we observe the performance variations of CV-DMH with $\alpha$, $\beta$, and $\gamma$. They are used in Eq.(\ref{disobj}) to play the trade-off between regularization terms and empirical loss. Due to the limited space, we only report the results on \emph{Paris5K} when hashing code length is 128. Similar results can be found on other code lengths and datasets. We test the results when $\alpha$, $\beta$, $\gamma$ are varied from $[10^{-4}, 10^{-2}, 1, 10^2, 10^4]$. We observe performance variations with two of them by fixing the other parameter. Experimental results are presented in Fig.\ref{fig_robustness}(a-c). From these figures, we can clearly find that the performance is relatively stable to a wide range of $\alpha, \beta, \gamma$ variations.

\section{Conclusion and Future Work}
\label{sec:6}
In this paper, we propose a novel hashing scheme CV-DMH to learn compact hashing codes for supporting efficient and robust mobile landmark search. The design of CV-DMH has inspired by an interesting observation that: only canonical views of landmarks are frequently photographed and disseminated by different tourists, these views naturally provide effective visual representation basis of landmark hashing. We first develop submodular function based iterative mining to select canonical views that are discriminative and with limited redundancy. An intermediate representation is then generated on canonical views to robustly characterize the diverse visual landmark contents with sparse visual relations. At the final stage, we design a discrete binary embedding model to transform continuous intermediate representation to compact binary codes. To avoid relaxing quantization errors brought in conventional methods, we propose ALM-based optimization method to directly solve discrete solution. Experimental results on three real landmark datasets demonstrate that our proposed approach can achieve superior performance compared with several state-of-the-art approaches.

The current work will continue along several directions for further investigation:
\begin{enumerate}[1.]
    \item Recent works \cite{sigpro/ChengS16} indicate that contextual modalities associated with landmark images enjoy better discriminative capability. They inspire us to integrate contextual modalities with our hashing model and thus embed more discriminative semantics into hashing codes.
    \item Current three learning components of our approach are operated in three subsequent steps. In the further, we will try to develop a unified learning formulation to systemically integrate them together for further performance improvement.
    \item The Laplacian matrix construction based on visual graph in Eq.(\ref{disobj}) consumes $O(N^2)$. In the future, we will explore strategies to further reduce the computation from algorithmic perspective.
    \item We will apply our hashing model to other applications with the similar characteristics with landmarks.
\end{enumerate}

\ifCLASSOPTIONcaptionsoff
  \newpage
\fi

\section*{Acknowledgment}
The authors would like to thank the anonymous reviewers for their constructive and helpful suggestions.

\ifCLASSOPTIONcaptionsoff
  \newpage
\fi


\end{document}